\documentclass{article}

\usepackage{microtype}
\usepackage{graphicx}
\usepackage{caption}
\usepackage{subcaption}
\usepackage{booktabs} 

\usepackage{hyperref}

\usepackage[accepted]{icml2023}

\usepackage{amsmath}
\usepackage{amssymb}
\usepackage{mathtools}
\usepackage{amsthm}
\usepackage{bm}
\usepackage[algo2e,ruled,vlined]{algorithm2e}
\usepackage{soul}

\usepackage[capitalize,noabbrev]{cleveref}

\theoremstyle{plain}
\newtheorem{theorem}{Theorem}[section]
\newtheorem{proposition}[theorem]{Proposition}
\newtheorem{lemma}[theorem]{Lemma}

\theoremstyle{definition}

\theoremstyle{remark}

\usepackage[textsize=tiny]{todonotes}

\DeclareMathOperator*{\argmin}{argmin}

\newcommand{\E}{\mathbb{E}}
\newcommand{\R}{\mathbb{R}}
\newcommand{\Rd}{\mathbb{R}^d}

\newcommand{\cP}{\mathcal{P}}

\newcommand{\eps}{\varepsilon}
\newcommand{\dd}{{\mathrm{d}}}
\newcommand{\norm}[1]{\left\Vert#1\right\Vert}
\renewcommand{\phi}{\psi}
\newcommand*{\defeq}{\coloneqq}

\usepackage{bbm}
\usepackage{enumitem,kantlipsum}
\usepackage{float}

\newcommand{\ra}[1]{\renewcommand{\arraystretch}{#1}}
\usepackage{array}
\newcolumntype{L}[1]{>{\raggedright\let\newline\\\arraybackslash\hspace{0pt}}m{#1}}
\newcolumntype{C}[1]{>{\centering\let\newline\\\arraybackslash\hspace{0pt}}m{#1}}
\newcolumntype{R}[1]{>{\raggedleft\let\newline\\\arraybackslash\hspace{0pt}}m{#1}}

\usepackage{xspace}
\newcommand*{\eg}{{\it e.g.}\@\xspace}
\newcommand*{\ie}{{\it i.e.}\@\xspace}

\newcommand*{\tran}{^{\mkern-1.5mu\mathsf{T}}}

\newcommand{\Real}{\mathbb R}
\newcommand{\too}{\rightarrow}

\newcommand{\brac}[1]{\left [#1\right ]}
\newcommand{\bigO}{\mathcal{O}}

\definecolor{mygray}{gray}{0.95}

\icmltitlerunning{Multisample Flow Matching}

\begin{document}

\twocolumn[
\icmltitle{Multisample Flow Matching: Straightening Flows with Minibatch Couplings}



\icmlsetsymbol{equal}{*}

\begin{icmlauthorlist}
\icmlauthor{Aram-Alexandre Pooladian}{equal,meta,cds}
\icmlauthor{Heli Ben-Hamu}{equal,weizmann}
\icmlauthor{Carles Domingo-Enrich}{equal,meta,nyu}
\icmlauthor{Brandon Amos}{meta}
\icmlauthor{Yaron Lipman}{meta,weizmann}
\icmlauthor{Ricky T. Q. Chen}{meta}
\end{icmlauthorlist}

\icmlaffiliation{cds}{Center for Data Science, NYU}
\icmlaffiliation{nyu}{Courant Institute of Mathematical Sciences, NYU}
\icmlaffiliation{weizmann}{Weizmann Institute of Science}
\icmlaffiliation{meta}{Meta AI (FAIR)}

\icmlcorrespondingauthor{Ricky T. Q. Chen}{rtqichen@meta.com}

\icmlkeywords{Machine Learning, ICML}

\vskip 0.3in
]



\printAffiliationsAndNotice{\icmlEqualContribution} 

\begin{abstract}
Simulation-free methods for training continuous-time generative models construct probability paths that go between noise distributions and individual data samples. Recent works, such as Flow Matching, derived paths that are optimal for each data sample. However, these algorithms rely on independent data and noise samples, and do not exploit underlying structure in the data distribution for constructing probability paths. We propose Multisample Flow Matching, a more general framework that uses non-trivial couplings between data and noise samples while satisfying the correct marginal constraints. At very small overhead costs, this generalization allows us to (i) reduce gradient variance during training, (ii) obtain straighter flows for the learned vector field, which allows us to generate high-quality samples using fewer function evaluations, and (iii) obtain transport maps with lower cost in high dimensions, which has applications beyond generative modeling. Importantly, we do so in a completely simulation-free manner with a simple minimization objective. We show that our proposed methods improve sample consistency on downsampled ImageNet data sets, and lead to better low-cost sample generation.\looseness=-1
\end{abstract}

\section{Introduction}

\begin{figure}[t]
    \centering
    \begin{subfigure}[b]{0.49\linewidth}
        {\tiny \; NFE=400 \;\;\;\;\;\; 12 \;\;\;\;\;\;\;\;\;\;\;\;\; 8 \;\;\;\;\;\;\;\;\;\;\;\;\;\; 6} \\
        \includegraphics[width=\linewidth]{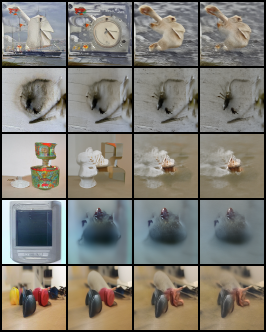}
        \vspace{-1.5em}
        \caption*{Flow Matching}
    \end{subfigure}
    \begin{subfigure}[b]{0.49\linewidth}
        {\tiny \; NFE=400 \;\;\;\;\;\; 12 \;\;\;\;\;\;\;\;\;\;\;\;\; 8 \;\;\;\;\;\;\;\;\;\;\;\;\;\; 6} \\
        \includegraphics[width=\linewidth]{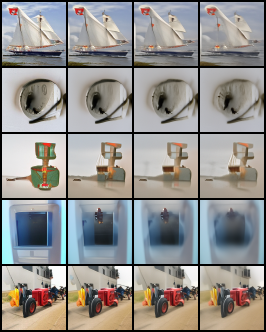}
        \vspace{-1.5em}
        \caption*{Multisample Flow Matching}
    \end{subfigure}
    \caption{Multisample Flow Matching trained with batch optimal couplings produces more consistent samples across varying NFEs. Note that both flows on each row start from the same noise sample.}
    \label{fig:samples_vs_nfe_imagenet64}
\end{figure}

Deep generative models offer an attractive family of
paradigms that can approximate a data distribution and produce high quality samples, with impressive results in recent years \citep{ramesh2022hierarchical,saharia2022photorealistic,gafni2022make}. 
In particular, these works have made use of simulation-free training methods for diffusion models \citep{ho2020denoising,song2020score}. 
A number of works have also adopted and generalized these simulation-free methods \citep{lipman2022flow,albergo2022building,liu2022flow,neklyudov2022action} for continuous normalizing flows (CNF; \citet{chen2018neural}), a family of continuous-time deep generative models that parameterizes a vector field which flows noise samples into 
data samples. 

Recently, \citet{lipman2022flow} proposed \emph{Flow Matching} (FM), a method to train CNFs based on constructing explicit \emph{conditional probability paths} between the noise distribution (at time $t=0$) and each data sample (at time $t=1$). Furthermore, they showed that these conditional probability paths can be taken to be the optimal transport path when the noise distribution is a standard Gaussian, a typical assumption in generative modeling. However, this does not imply that the \emph{marginal probability path} (marginalized over the data distribution) is anywhere close to the optimal transport path between the noise and data distributions.

Most existing works, including diffusion models and Flow Matching, have only considered conditional sample paths where the endpoints (a noise sample and a data sample) are sampled independently. However, this results in non-zero gradient variances even at convergence, slow training times, and in particular limits the design of probability paths. In turn, it becomes difficult to create paths that are fast to simulate, a desirable property for both likelihood evaluation and sampling.

\textbf{Contributions:} We present a tractable instance of Flow Matching with joint distributions, which we call \emph{Multisample Flow Matching}. Our proposed method generalizes the construction of probability paths by considering non-independent couplings of $k$-sample empirical distributions. 

Among other theoretical results, we show that if an appropriate optimal transport (OT) inspired coupling is chosen, then sample paths become straight as the batch size $k\to\infty$, leading to more efficient simulation. In practice, we observe both improved sample quality on ImageNet using adaptive ODE solvers and using simple Euler discretizations with a low budget number of function evaluations. 
Empirically, we find that on ImageNet, we can \emph{reduce the required sampling cost by 30\% to 60\%} for achieving a low Fr{\'e}chet Inception Distance (FID) compared to a baseline Flow Matching model, while introducing only 4\% more training time. This improvement in sample efficiency comes at no degradation in performance, \eg log-likelihood and sample quality. 

Within the deep generative modeling paradigm, this allows us to regularize towards the optimal vector field in a \textit{completely simulation-free manner} (unlike \eg \citet{finlay2020train,liu2022flow}), and avoids adversarial formulations (unlike \eg \citet{makkuva2020optimal,albergo2022building}). 
In particular, we are the first work to be able to make use of solutions from optimal solutions on minibatches while preserving the correct marginal distributions, whereas prior works would only fit to the barycentric average (see detailed discussion in \cref{sec:related_batchot}). Beyond generative modeling, we also show how our method can be seen as a new way to compute approximately optimal transport maps between arbitrary distributions in settings where the cost function is completely unknown and only minibatch optimal transport solutions are provided.

\section{Preliminaries}

\subsection{Continuous Normalizing Flow}\label{sec: cnf}
Let $\Real^d$ denote the data space with data points $x=(x^1,\ldots,x^d) \in \Real^d$. Two important objects we use in this paper are: the \emph{probability path} $p_t: \Real^d \too \Real_{>0}$, which is a time dependent (for $t \in [0,1]$) probability density function, \ie, $\int p_t(x)dx = 1$, and a \emph{time-dependent vector field}, $u_t:[0,1]\times \Real^d \too \Real^d$. A vector field $u_t$ constructs a time-dependent diffeomorphic map, called a \emph{flow}, $\phi:[0,1]\times \Real^d \too \Real^d$, defined via the ordinary differential equation (ODE):\looseness=-1
\begin{align}\label{e:ode}
    \frac{d}{dt}\phi_t(x_0) &= u_t(\phi_t(x_0))\,, \quad \phi_0(x_0) = x_0\,.
\end{align}
To create a deep generative model, \citet{chen2018neural} suggested modeling the vector field $u_t$ with a neural network, 
leading to a deep parametric model of the flow $\phi_t$, referred to as a \emph{Continuous Normalizing Flow} (CNF). A CNF is often used to transform a density $p_0$ to a different one, $p_1$, via the push-forward equation
\begin{equation}\label{e:push_forward}
\!\!\! p_t(x) = [\phi_t]_\sharp p_0(x) = p_0(\phi_t^{-1}(x)) \left| \det \brac{ \frac{\partial \phi_t^{-1}}{\partial x}(x)} \right| ,
\end{equation}
where the second equality defines the push-forward (or change of variables) operator $\sharp$.
A vector field $u_t$ is said to \emph{generate} a probability path $p_t$ if its flow $\phi_t$ satisfies \eqref{e:push_forward}. 

\subsection{Flow Matching}\label{sec: flowmatching} A simple simulation-free method for training CNFs is the \emph{Flow Matching} algorithm~\citep{lipman2022flow}, which regresses onto an (implicitly-defined) target vector field that generates the desired probability density path $p_t$. Given two marginal distributions $q_0(x_0)$ and $q_1(x_1)$ for which we would like to learn a CNF to transport between, Flow Matching seeks to optimize the simple regression objective,
\begin{equation}\label{eq:fm}
    \E_{t, p_t(x)} \norm{ v_t(x;\theta) - u_t(x) }^2,
\end{equation}
where $v_t(x;\theta)$ is the parametric vector field for the CNF, and $u_t(x)$ is a vector field that generates a probability path $p_t$ under the two marginal constraints that $p_{t=0} = q_0$ and $p_{t=1} = q_1$. While \cref{eq:fm} is the ideal objective function to optimize, not knowing $(p_t,u_t)$ makes this computationally intractable.

\citet{lipman2022flow} proposed a tractable method of optimizing \eqref{eq:fm}, which first defines \emph{conditional} probability paths and vector fields, such that when marginalized over $q_0(x_0)$ and $q_1(x_1)$, provide both $p_t(x)$ and $u_t(x)$. 
When targeted towards generative modeling, $q_0(x_0)$ is a simple noise distribution and easy to directly enforce, leading to a one-sided construction:
\begin{align}
    \label{eq:marg_prob}
    p_t(x) &= \int p_t(x | x_1) q_1(x_1) \;d x_1 \\
    \label{eq:marg_vf}
    u_t(x) &= \int u_t(x | x_1) \frac{p_t(x | x_1) q_1(x_1)}{p_t(x)} \;d x_1,
\end{align}
where the conditional probability path is chosen such that 
\begin{equation}\label{eq:cond_prob_conditions}
    p_{t=0}(x | x_1) = q_0(x) \text{\;\; and \;\;} p_{t=1}(x | x_1) = \delta(x - x_1),
\end{equation}
where $\delta(x-a)$ is a Dirac mass centered at $a \in \Rd$. By construction, $p_t(x|x_1)$ now satisfies both marginal constraints.

\citet{lipman2022flow} shows that if $u_t(x | x_1)$ generates $p_t(x | x_1)$, then the marginalized $u_t(x)$ generates $p_t(x)$, and furthermore, one can train using the much simpler objective of \emph{Conditional Flow Matching} (CFM):
\begin{equation}\label{eq:cfm}
\E_{t, q_1(x_1), p_t(x | x_1)} \norm{ v_t(x;\theta) - u_t(x_t | x_1) }^2,
\end{equation}
with $x_t = \psi_t(x_0|x_1)$; see \ref{sec: condot} for more details.
Note that this objective has the same gradient with respect to the model parameters $\theta$ as Eq. \eqref{eq:fm} ~\citep[Theorem 2]{lipman2022flow}.

\subsubsection{Conditional OT (CondOT) path}\label{sec: condot} One particular choice of conditional path $p_t(x | x_1)$ is to use the flow that corresponds to the optimal transport displacement interpolant \citep{mccann1997convexity} when $q_0(x_0)$ is the standard Gaussian, a common convention in generative modeling.
The vector field that corresponds to this is
\begin{equation}\label{eq:cond_vf}
    u_t(x_t | x_1) =  \frac{x_1 - x}{1 - t}.
\end{equation}
Using this conditional vector field in \eqref{e:ode}, this gives the conditional flow 
\begin{equation}\label{eq:cond_flow}
x_t = \phi_t(x_0 | x_1) =  (1 - t)x_0 + t x_1\,.
\end{equation}
Substituting \eqref{eq:cond_flow} into \eqref{eq:cond_vf}, one can also express the value of this vector field using a simpler expression,
\begin{equation}
\label{eq:cond_vf_simplified}
    u_t( x_t | x_1) = x_1 - x_0\,.
\end{equation} 
It is evident that this results in conditional flows that \textit{(i)} tranports all points $x_0$ from $t=0$ to $x_1$ at exactly $t=1$ and \textit{(ii)} are straight paths between the samples $x_0$ and $x_1$. This particular case of straight paths was also studied by \citet{liu2022flow} and \citet{albergo2022building}, where the conditional flow \eqref{eq:cond_flow} is referred to as a stochastic interpolant.
\citet{lipman2022flow} additionally showed that the conditional construction can be applied to a large class of Gaussian conditional probability paths, 
namely when $p_t(x | x_1) = \mathcal{N}(x | \mu_t(x_1), \sigma_t(x_1)^2I )$. This family of probability paths encompasses most prior diffusion models where probability paths are induced by simple diffusion processes with linear drift and constant diffusion~(\eg \citet{ho2020denoising,song2020score}). However, existing works mostly consider settings where $q_0(x_0)$ and $q_1(x_1)$ are sampled independently when computing training objectives such as \eqref{eq:cfm}.

\subsection{Optimal Transport: Static \& Dynamic}\label{sec: ot} 
Optimal transport generally considers methodologies that define some notion of distance on the space of probability measures \cite{villani2008optimal,villani2003topics,San15}. 
Letting $\cP(\Rd)$ be the space of probability measures over $\Rd$, we define the Wasserstein distance with respect to a cost function $c : \Rd \times \Rd \to \R_+$  between two measures $q_0,q_1 \in \cP(\Rd)$ as \citep{Kan42}
\begin{align}\label{eq: ot_linear}
    W_c(q_0,q_1) \defeq \min_{q \in \Gamma(q_0,q_1)}\E_{q(x_0,x_1)}[c(x_0, x_1)]\,,
\end{align}
where $\Gamma(q_0,q_1)$ is the set of joint measures with left marginal equal to $q_0$ and right marginal equal to $q_1$, called the set of \emph{couplings}. The minimizer to \cref{eq: ot_linear} is called the optimal coupling, which we denote by $q^*_c$. In the case where $c(x_0,x_1) \defeq \|x_0 - x_1\|^2$, the squared-Euclidean distance, \cref{eq: ot_linear} amounts to the (squared) $2$-Wasserstein distance $W_2^2(q_0,q_1)$, and we simply write the optimal transport plan as $q^*$.

Considering again the squared-Euclidean cost, in the case where $q_0$ exhibits a density over $\Rd$ (e.g. if $q_0$ is the standard normal distribution), \citet{benamou2000computational} states that $W_2^2(q_0,q_1)$ can be equivalently expressed as a \emph{dynamic} formulation,\looseness=-1
\begin{equation}\label{eq:dyn_ot}
    W_2^2(q_0,q_1) = \min_{p_t, u_t} \int_{0}^1 \int_{\Rd} \norm{u_t(x)}^2 p_t(x) \dd x_0 \dd t.
\end{equation}
where $u_t$ generates $p_t$, and $p_t$ satisfies boundary conditions $p_{t=0} = q_0$ and $p_{t=1} = q_1$. The optimality condition ensures that sample paths $x_t$ are straight lines, i.e. minimize the length of the path, and leads to paths that are much easier to simulate. Some prior approaches have sought to regularize the model using this optimality objective (\eg \citet{tong2020trajectorynet,finlay2020train}).
In contrast, instead of directly minimizing \eqref{eq:dyn_ot}, we will discuss an approach based on using solutions of the optimal coupling $q^*$ on minibatch problems, while leaving the marginal constraints intact. 

\section{Flow Matching with Joint Distributions} \label{sec:fm_joint}

While Conditional Flow Matching in \eqref{eq:cfm} leads to an unbiased gradient estimator for the Flow Matching objective, it was designed with independently sampled $x_0$ and $x_1$ in mind. 
We generalize the framework from Subsection~\ref{sec: flowmatching} to a construction that uses arbitrary joint distributions of $q(x_0, x_1)$ which satisfy the correct marginal constraints, \ie 
\begin{align} \label{eq:q_marginals}
   \!\!\!\! \int \!\! q(x_0, x_1) \dd x_1 \!=\!q_0(x_0)\,, \ \! \int \!\! q(x_0, x_1) \dd x_0 \!=\!q_1(x_1). 
\end{align}
We will show in Subsection~\ref{sec:multisample_fm} that this can potentially lead to lower gradient variance during training and allow us to design more optimal marginal vector fields $u_t(x)$ with desirable properties such as improved sample efficiency.

Building on top of Flow Matching, we propose modifying the conditional probability path construction \eqref{eq:cond_prob_conditions} so that at $t=0$, we define
\begin{equation}
    p_{t=0}(x_0 | x_1) = q(x_0 | x_1).
\end{equation}
where $q(x_0 | x_1)$ is the conditional distribution $\tfrac{q(x_0, x_1)}{q_1(x_1)}$.
Using this construction, we still satisfy the marginal constraint, 
$$p_{0}(x) = \int p_{0}(x|x_1)q_1(x_1)dx_1 =\int q(x, x_1) dx_1 = q_0(x)$$
\ie $p_{t=0}(x) = \int q(x, x_1) dx_1 = q_0(x)$ by the assumption made in \eqref{eq:q_marginals}. Then similar to \citet{chen2023riemannian}, we note that the conditional probability path $p_t(x | x_1)$ \emph{need not be explicitly formulated} for training, and that only an appropriate conditional vector field $u_t(x | x_1)$ needs to be chosen such that all points arrive at $x_1$ at $t=1$, which ensures $p_{t=1}(x|x_1) = \delta(x - x_1)$. As such, we can make use of the same conditional vector field as prior works, \eg the choice in  \cref{eq:cond_vf,eq:cond_flow,eq:cond_vf_simplified}.

We then propose the \textbf{Joint CFM} objective as
\begin{center}\vspace{-1.5em}			
    \colorbox{mygray} {		
      \begin{minipage}{0.977\linewidth} 	
       \centering
       \vspace{-1em}
    \begin{equation}\label{eq:cfm_joint}
    \mathcal{L}_\text{JCFM} = \E_{t, q(x_0, x_1)} \norm{ v_t(x_t;\theta) - u_t( x_t | x_1) }^2,
    \end{equation}   
      \end{minipage}}			
      \vspace{-1em}
\end{center}
where $x_t = \phi_t(x_0 | x_1)$ is the conditional flow. Training only involves sampling from $q(x_0, x_1)$ and does not require explicitly knowing the densities of $q(x_0, x_1)$ or $p_t(x | x_1)$.
Note that Equation \eqref{eq:cfm_joint} reduces to the original CFM objective \eqref{eq:cfm} when $q(x_0,x_1) = q_0(x_0) q_1(x_1)$.

A quick sanity check shows that this objective can be used with any choice of joint distribution $q(x_0, x_1)$.
\begin{lemma} \label{lem:marginals}
    The optimal vector field $v_t(\cdot;\theta)$ in \eqref{eq:cfm_joint}, which is the marginal vector field $u_t$, maps between the marginal distributions $q_0(x_0)$ and $q_1(x_1)$.
\end{lemma}
In the remainder of the section, we highlight some motivations for using joint distributions $q(x_0, x_1)$ that are different from the independent distribution $q_0(x_0) q_1(x_1)$.

\paragraph{Variance reduction} Choosing a good joint distribution can be seen as a way to reduce the variance of the gradient estimate, which improves and speeds up training.
We develop the gradient covariance at a fixed $x$ and $t$, and bound its total variance:
\begin{lemma} \label{lem:var_bound}
The total variance (i.e. the trace of the covariance) of the gradient at a fixed $x$ and $t$ is bounded as:
\begin{align} \label{eq:bound_tr_cov}
&\sigma^2_{t,x}=\mathrm{Tr}\big[\mathrm{Cov}_{p_t(x_1|x)} \left( \nabla_\theta \norm{v_t(x;\theta) - u_t(x | x_1)}^2 \right)\big]\\ 
&\leq \|\nabla_\theta v_t(x;\theta)\|^2 \, \mathbb{E}_{p_t(x_1|x)} \|u_t(x) - u_t(x | x_1) \|^2 \nonumber
\end{align}
Then $\E_{t, p_t(x)}[\sigma^2_{t,x}]$ is bounded above by:
\begin{align} 
\label{eq:avg_total_variance}
    \max_{t,x} \norm{\nabla_\theta v_t(x;\theta)}^2 \times \mathcal{L}_\mathrm{JCFM}
\end{align}
\end{lemma}
This proves that $\E_{t, p_t(x)}[\sigma^2_{t,x}]$, which is the average gradient variance at fixed $x$ and $t$, is upper bounded in terms of the Joint CFM objective. That means that minimizing the Joint CFM objective help in decreasing $\E_{t, p_t(x)}[\sigma^2_{t,x}]$. Note also that $\E_{t, p_t(x)}[\sigma^2_{t,x}]$ is not the gradient variance and is always smaller, as it does not account for variability over $x$ and $t$, but it is a good proxy for it. The proof is in App.~\ref{subsec:grad_variance}. 

Sampling $x_0$ and $x_1$ independently generally cannot achieve value zero for $\E_{t, p_t(x)}[\sigma^2_{t,x}]$ 
\emph{even at the optimum}, since there are an infinite number of pairs $(x_0,x_1)$ whose conditional path crosses any particular $x$ at a time $t$.
As shown in \eqref{eq:avg_total_variance}, having a low optimal value for the Joint CFM objective is a good proxy for low gradient variance and hence a desirable property for choosing a joint distribution $q(x_0,x_1)$. In \cref{sec:multisample_fm}, we show that certain joint distributions have optimal Joint CFM values close to zero.

\paragraph{Straight flows} Ideally, the flow $\phi_t$ of the marginal vector field $u_t$ (and of the learned $v_{\theta}$ by extension) should be close to a straight line. The reason is that ODEs with straight trajectories can be solved with high accuracy using fewer steps (i.e. function evaluations), which speeds up sample generation. The quantity  
\begin{align}\label{eq: straightness}
    S = \mathbb{E}_{t, q_0(x_0)} \big[ \|u_t( \phi_t(x_0))\|^2 - \| \phi_1(x_0) - x_0\|^2 \big],
\end{align}
which we call the \emph{straightness} of the flow and was also studied by \citet{liu2022rectified}, measures how straight the trajectories are. Namely, we can rewrite it as 
\begin{align} \label{eq:straightness_2}
    \!\!\! S = \mathbb{E}_{t, q_0(x_0)} \left[ \|u_t( \phi_t(x_0)) - \mathbb{E}_{t'} \left[ u_{t'}( \phi_{t'}(x_0)) \right] \|^2 \right],
\end{align}
which shows that $S \geq 0$ and only zero if $u_t( \phi_t(x_0))$ is constant along $t$, which is equivalent to $\phi_t(x_0)$ being a straight line.\looseness=-1

When $x_0$ and $x_1$ are sampled independently, the straightness is in general far from zero. This can be seen in the CondOT plots in Figure~\ref{fig:2D_1} (right); if flows were close to straight lines, samples generated with one function evaluation (NFE=1) would be of high quality. In \cref{sec:multisample_fm}, we show that for certain joint distributions, the straightness of the flow is close to zero. 

\paragraph{Near-optimal transport cost} By Lemma~\ref{lem:marginals}, the flow $\phi_t$ corresponding to the optimal $u_t$ satisfies that $\phi_0(x_0) = x_0 \sim q_0$ and $\phi_1(x_0) \sim q_1$. Hence, $x_0 \mapsto \phi_1(x_0)$ is a transport map between $q_0$ and $q_1$ with an associated transport cost \looseness=-1
\begin{align} \label{eq:transport_cost_phi_1}
\E_{q_0(x_0)} \|\phi_1(x_0) - x_0\|^2.
\end{align}
There is no reason to believe that when $x_0$ and $x_1$ are sampled independently, the transport cost $\E_{q_0(x_0)} \|\phi_1(x_0) - x_0\|^2$ will be anywhere near the optimal transport cost $W_2^2(p_0,p_1)$. Yet, in Section~\ref{sec:multisample_fm} we show that for well chosen $q$, the transport cost for $\phi_1$ does approach its optimal value. Computing optimal (or near-optimal) transport maps in high dimensions is a challenging task \cite{makkuva2020optimal,amos2022amortizing} that extends beyond generative modeling and into the field of optimal transport, and it has applications in computer vision \cite{feydy2017optimal,SolGoePey15,SolPeyKim16,liu20232} and computational biology \cite{lubeck2022neural,bunne2021learning,bunne2022supervised,schiebinger2019optimal}, for instance. Hence, Joint CFM may also be viewed as a practical way to obtain approximately optimal transport maps in this context.\looseness=-1

\section{Multisample Flow Matching}\label{sec:multisample_fm}

Constructing a joint distribution satisfying the marginal constraints is difficult, especially since at least one of the marginal distributions is based on empirical data. 
We thus discuss a method to construct the joint distribution $q(x_0, x_1)$ implictly by designing a suitable sampling procedure that leaves the marginal distributions invariant. Note that training with \eqref{eq:cfm_joint} only requires sampling from $q(x_0, x_1)$.

We use a multisample construction for $q(x_0, x_1)$ in the following manner:\looseness=-1
\begin{center}\vspace{-1.5em}			
    \colorbox{mygray} {		
      \begin{minipage}{0.977\linewidth} 
       \centering
\begin{enumerate}[wide, labelwidth=!, labelindent=0pt]
    \item Sample $\smash{\{x_0^{(i)}\}_{i=1}^k \sim q_0(x_0)}$ and $\smash{\{x_1^{(i)}\}_{i=1}^k \sim q_1(x_1)}$.
    \label{item:1}
    \item Construct a doubly-stochastic matrix with probabilities $\pi(i, j)$ dependent on the samples $\smash{\{x_0^{(i)}\}_{i=1}^k}$ and $\smash{\{x_1^{(i)}\}_{i=1}^k}$.
    \item Sample from the discrete distribution,\\ $\smash{q^k(x_0, x_1) = \frac{1}{k}
    \sum_{i,j=1}^k \delta(x_0 - x_0^{i}) \delta(x_1 - x_1^{j}) \pi(i, j)}$.
    \label{item:3}
\end{enumerate} 
      \end{minipage}}			 
      \vspace{-0.75em}
\end{center}
Marginalizing $q^k(x_0, x_1)$ over samples from Step 1, we obtain the implicitly defined $q(x_0, x_1)$. By choosing different \emph{couplings} $\pi(i, j)$, we induce different joint distributions. In this work, we focus on couplings that induce joint distributions which approximates, or at least partially satisfies, the optimal transport joint distribution. The following result, proven in App.~\ref{subsec:q_marginals}, guarantees that $q$ has the right marginals.\looseness=-1
\begin{lemma} \label{lem:q_marginals}
    The joint distribution $q(x_0,x_1)$ constructed in Steps [\ref{item:1}-\ref{item:3}] has marginals $q_0(x_0)$ and $q_1(x_1)$.
\end{lemma}
That is, the marginal constraints \eqref{eq:q_marginals} are satisfied and consequently we are allowed to use the framework of Section~\ref{sec:fm_joint}.

\subsection{CondOT is Uniform Coupling}\label{sec: condot_unif}
The aforementioned multisample construction subsumes the independent joint distribution used by prior works, when the joint coupling is taken to be uniformly distributed, \ie $\pi(i, j) = \frac{1}{k}$. This is precisely the coupling used by \cite{lipman2022flow} under our introduced notion of Multisample Flow Matching, and acts as a natural reference point.

\begin{figure*}
    \centering
    \begin{minipage}[c]{0.730\textwidth}
    \centering
        \rotatebox{90}{\hspace{0.4em}\tiny Diffusion}
        \includegraphics[width=0.98\textwidth]{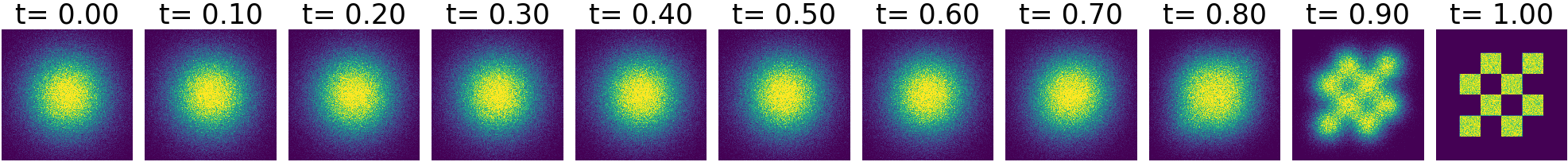} \\
        \rotatebox{90}{\hspace{0.5em}\tiny CondOT}
        \includegraphics[width=0.98\textwidth, trim=0 0 0 30px, clip]{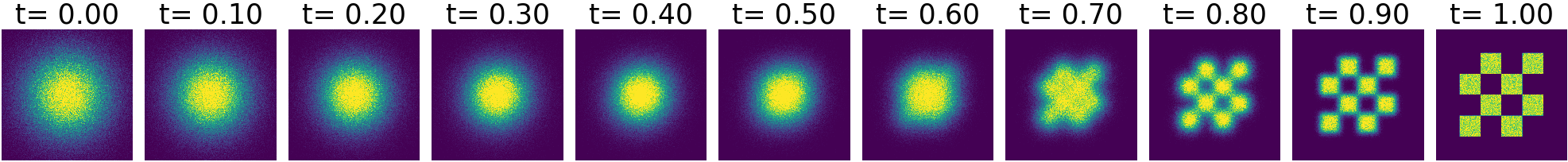} \\
        \rotatebox{90}{\hspace{0.7em}\tiny Stable}
        \includegraphics[width=0.98\textwidth]{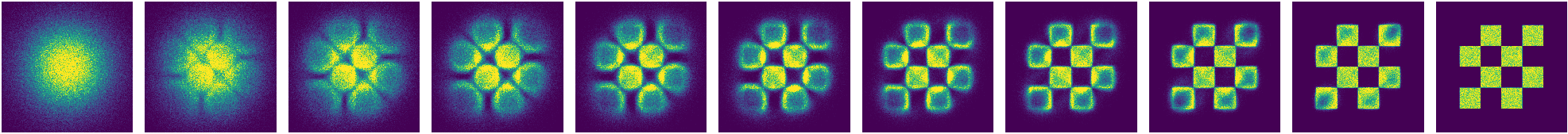} \\
        \rotatebox{90}{\hspace{0.5em}\tiny Heuristic}
        \includegraphics[width=0.98\textwidth]{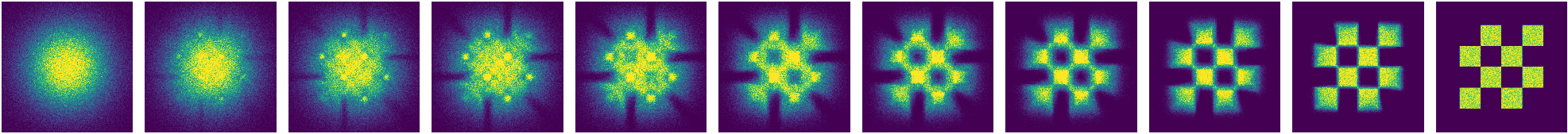} \\ 
        \rotatebox{90}{\hspace{0.3em}\tiny BatchEOT}
        \includegraphics[width=0.98\textwidth]{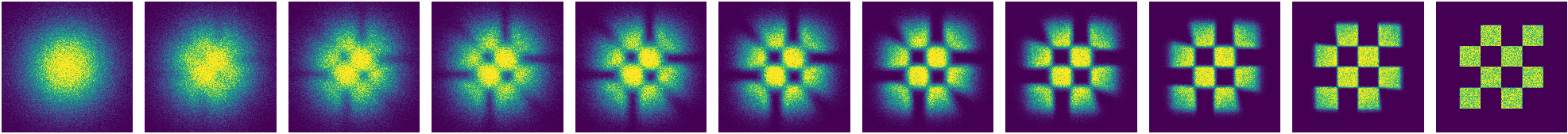} \\
        \rotatebox{90}{\hspace{0.5em}\tiny BatchOT}
        \includegraphics[width=0.98\textwidth]{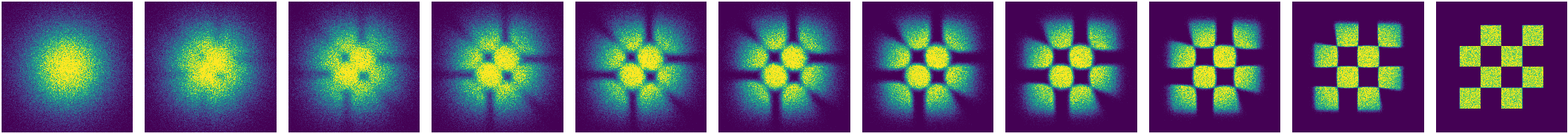}
    \end{minipage}
    \hfill\vline\hfill
    \begin{minipage}[c]{0.255\textwidth}
    \centering
        \includegraphics[width=\textwidth]{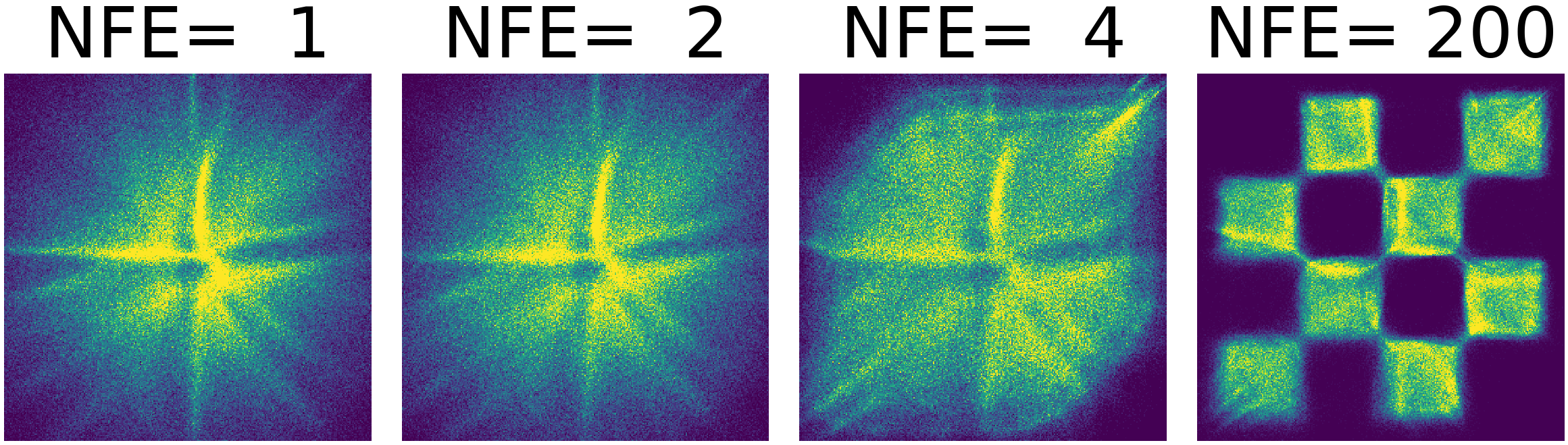} \\
        \includegraphics[width=\textwidth, trim=0 0 0 30px, clip]{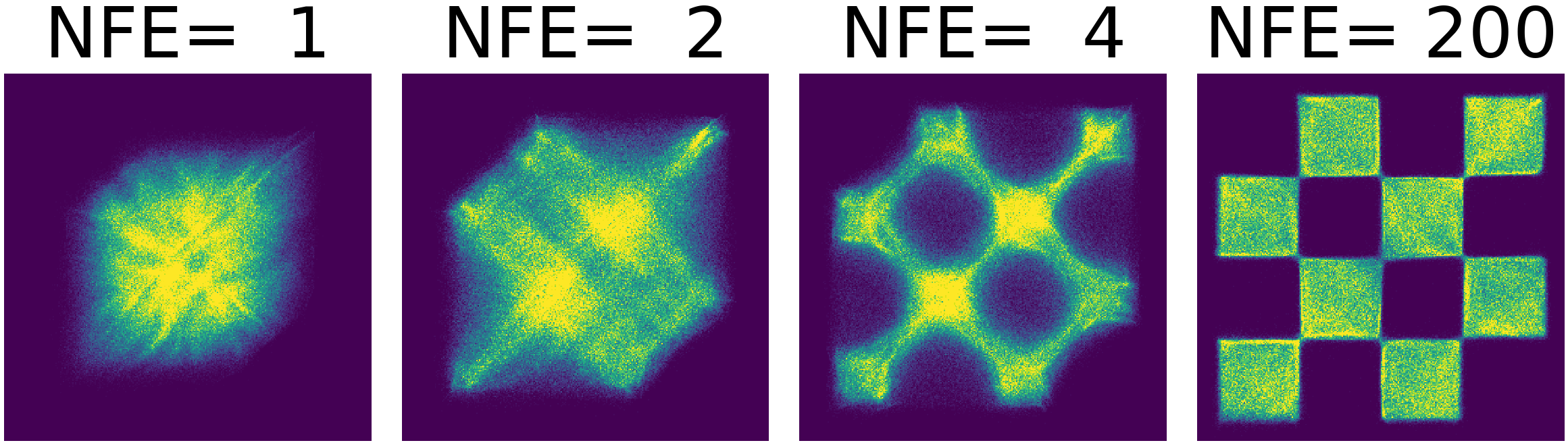} \\
        \includegraphics[width=\textwidth]{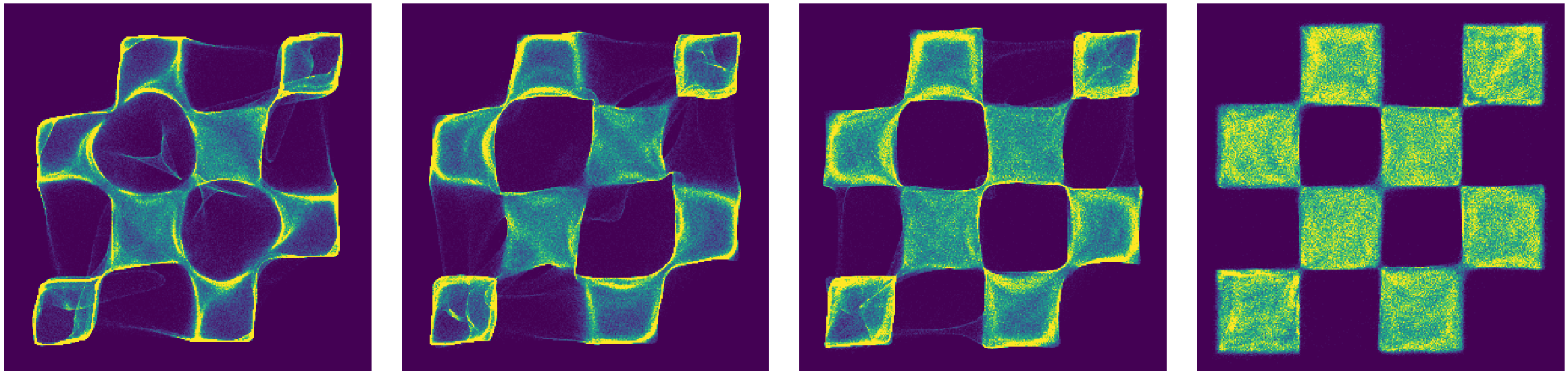} \\
        \includegraphics[width=\textwidth]{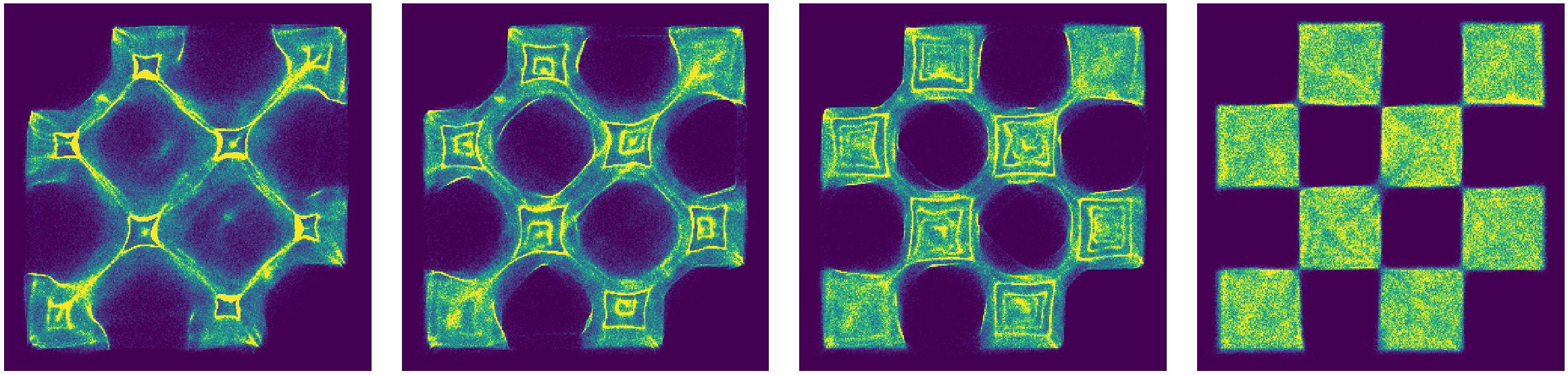} \\
        \includegraphics[width=\textwidth]{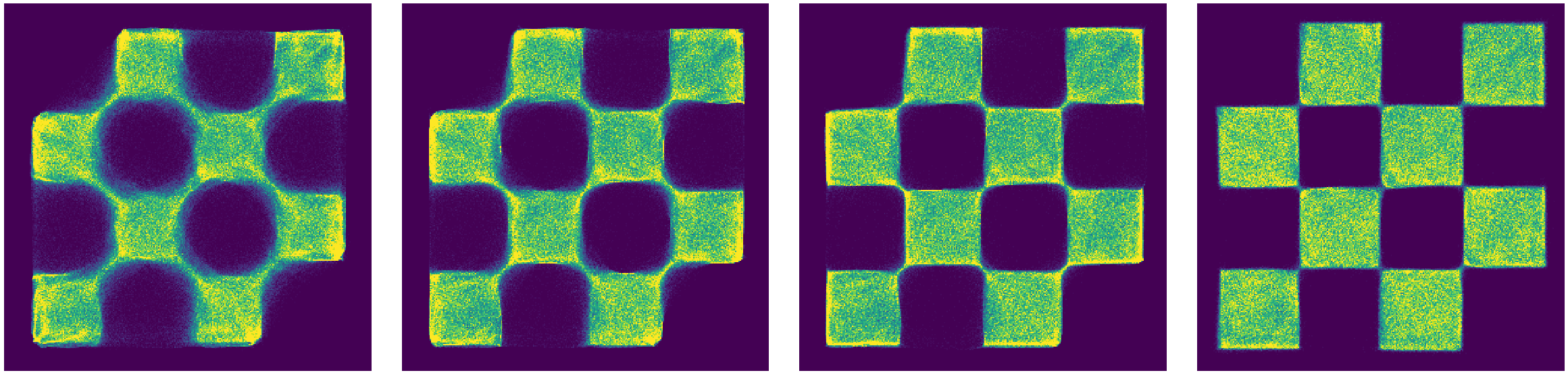} \\
        \includegraphics[width=\textwidth]{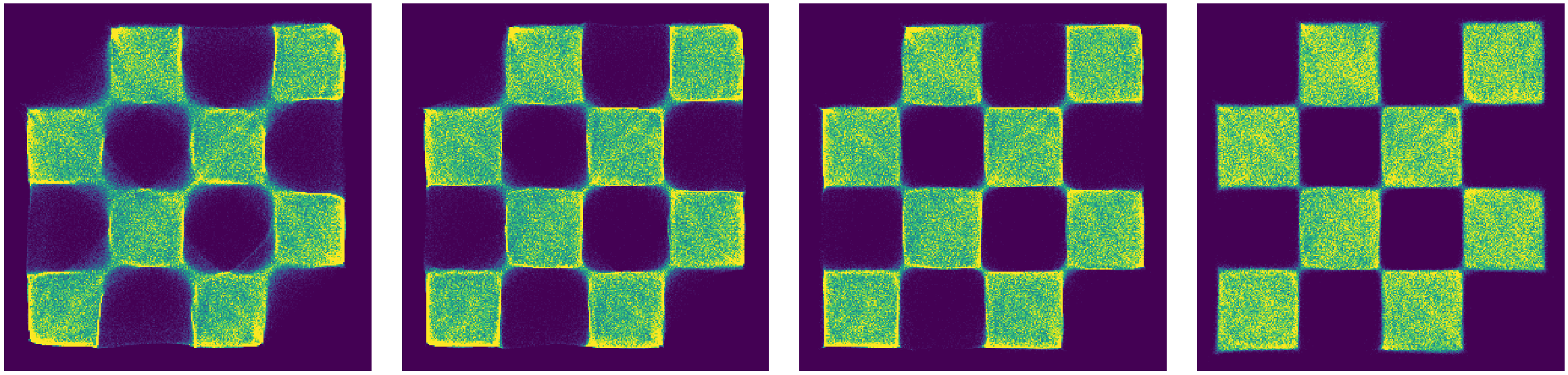}
    \end{minipage}
    \caption{Multisample Flow Matching learn probability paths that are much closer to an optimal transport path than baselines such as Diffusion and CondOT paths. (\emph{Left}) Exact marginal probability paths. (\emph{Right}) Samples from trained models at $t=1$ for different numbers of function evaluations (NFE), using Euler discretization.
    Furthermore, the final values of the Joint CFM objective \eqref{eq:cfm_joint}---upper bounds on the variance of $u_t$ at convergence---are: CondOT: 10.72; Stable: 1.60, Heuristic: 1.56; BatchEOT: 0.57, BatchOT: 0.24.}
    \label{fig:2D_1}
\end{figure*}

\subsection{Batch Optimal Transport (BatchOT) Couplings} \label{sec:batch_ot}
The natural connections between optimal transport theory and optimal sampling paths in terms of straight-line interpolations, lead us to the following pseudo-deterministic coupling, which we call Batch Optimal Transport (BatchOT).
While it is difficult to solve \eqref{eq: ot_linear} at the population level, it can efficiently solved on the level of samples. Let $\smash{\{x_0^{(i)}\}_{i=1}^k} \sim q_0(x_0)$ and $\smash{\{x_1^{(i)}\}_{i=1}^k} \sim q_1(x_1)$. When defined on batches of samples, the OT problem \eqref{eq: ot_linear} can be solved exactly and efficiently using standard solvers, as in \texttt{POT} \citep[Python Optimal Transport]{flamary2021pot}. On a batch of $k$ samples, the runtime complexity is well-understood via either the Hungarian algorithm or network simplex algorithm, with an overall complexity of $\mathcal{O}(k^3)$ \citep[Chapter 3]{PeyCut19}. The resulting coupling $\pi^{k,*}$ from the algorithm is a \textit{permutation matrix}, which is a type of doubly-stochastic matrix that we can incorporate into Step \ref{item:3} of our procedure. 

We consider the effect that the sample size $k$ has on the marginal vector field $u_t(x)$. The following theorem shows that in the limit of $k \to \infty$, BatchOT satisfies the three criteria that motivate Joint CFM: variance reduction, straight flows, and near-optimal transport cost.

\begin{theorem}[Informal]
\label{thm:limiting}
    Suppose that Multisample Flow Matching is run with BatchOT. Then, as $k \to \infty$,
    \begin{enumerate}[label=(\roman*),noitemsep,topsep=0pt]
        \item The value of the Joint CFM objective (Equation \eqref{eq:cfm_joint}) for the optimal $u_t$ converges to 0.
        \item The straightness $S$ for the optimal marginal vector field $u_t$ (Equation \eqref{eq: straightness}) converges to zero.
        \item The transport cost $\E_{q_0(x_0)} \|\phi_1(x_0) - x_0\|^2$ (Equation \eqref{eq:transport_cost_phi_1}) associated to $u_t$ converges to the optimal transport cost $W_2^2(p_0,p_1)$.  
    \end{enumerate}
\end{theorem}
As $k \rightarrow \infty$, result \textit{(i)} implies that the gradient variance both during training and at convergence is reduced due to \cref{eq:avg_total_variance}; result \textit{(ii)} implies the optimal model will be easier to simulate between $t$=0 and $t$=1; result \textit{(iii)} implies that Multisample Flow Matching can be used as a simulation-free algorithm for approximating optimal transport maps.

The full version of Thm.~\ref{thm:limiting} can be found in App.~\ref{sec:proofs}, and it makes use of standard, weak technical assumptions which are common in the optimal transport literature. While Thm.~\ref{thm:limiting} only analyzes asymptotic properties, we provide theoretical evidence that the transport cost decreases with $k$, as summarized by a monotonicity result in Thm.~\ref{thm:monotone}.

\subsection{Batch Entropic OT (BatchEOT) Couplings}
For $k$ sufficiently large, the cubic complexity of the BatchOT approach is not always desirable, and instead one may consider approximate methods that produce couplings sufficiently close to BatchOT at a lower computational cost. A popular surrogate, pioneered in \cite{cuturi2013sinkhorn}, is to incorporate an entropic penalty parameter on the doubly stochastic matrix, pulling it closer to the independent coupling:
$$ \min_{q \in \Gamma(q_0,q_1)} \E_{(x_0,x_1)\sim q} \|x_0-x_1\|^2 + \eps H(q)\,,$$
where $H(q) = -\sum_{i,j} q_{i,j}(\log(q_{i,j}) - 1)$ is the entropy of the doubly stochastic matrix $q$, and $\eps > 0$ is some finite regularization parameter. The optimality conditions of this strictly convex program leads to 
Sinkhorn's algorithm, which has a runtime of $\tilde{\mathcal{O}}(k^2/\eps)$ \cite{AltWeeRig17}. 

The output of performing Sinkhorn's algorithm is a doubly-stochastic matrix.
The two limiting regimes of the regularization parameter are well understood (c.f. ~\citet{PeyCut19}, Proposition 4.1, for instance): as $\eps \to 0$, BatchEOT recovers the BatchOT permutation matrix from ~\cref{sec:batch_ot}; as $\eps \to \infty$, BatchEOT recovers the independent coupling on the indices from ~\cref{sec: condot_unif}.

\subsection{Stable and Heuristic Couplings}

An alternative approach is to consider faster algorithms that satisfy at least some desirable properties of an optimal coupling. In particular, an optimal coupling is \emph{stable}. A permutation coupling is stable if \emph{no pair of $x_0^{(i)}$ and $x_1^{(j)}$ favor each other over their assigned pairs based on the coupling.} Such a problem can be solved using the Gale-Shapeley algorithm \citep{gale1962college} which has a compute cost of $\bigO(k^2)$ given the cross set ranking of all samples. Starting from a random assignment, it is an iterative algorithm that reassigns pairs if they violate the stability property and can terminate very early in practice. Note that in a cost-based ranking, one has to sort the coupling costs of each sample with all samples in the opposing set, resulting in an overall $\bigO(k^2\log(k))$ compute cost. 

The Gale-Shapeley algorithm is agnostic to any particular costs, however, as stability is only defined in terms of relative rankings of individual samples. We design a modified version of this algorithm based on a heuristic for satisfying the cyclical monotonicity property of optimal transport, namely that should pairs be reassigned, the reassignment should not increase the total cost of already matched pairs. We refer to the output of this modified algorithm as a \emph{heuristic coupling} and discuss the details in Appendix \ref{app:heuristic_couplings}.

\section{Related Work}
Generative modeling and optimal transport are inherently intertwined topics, both often aiming to learn a transport between two distributions but with very different goals. Optimal transport is widely recognized as a powerful tool for large-scale generative modeling as it can be used to stabilize training \citep{arjovsky2017wasserstein}.
In the context of continuous-time generative modeling, optimal transport has been used to regularize continuous normalizing flows for easier simulation \citep{finlay2020train,onken2021ot}, and increase interpretability \citep{tong2020trajectorynet}. However, the existing methods for encouraging optimality in a generative model generally require either solving a potentially unstable min-max optimization problem (\eg \citep{arjovsky2017wasserstein,makkuva2020optimal,albergo2022building}) or require simulation of the learned vector field as part of training (\eg \citet{finlay2020train,liu2022flow}). In contrast, the approach of using batch optimal couplings can  be used to avoid the min-max optimization problem, but has not been successfully applied to generative modeling as they do not satisfy marginal constraints---we discuss this further in the following \cref{sec:related_batchot}. On the other hand, neural optimal transport approaches are mainly centered around the quadratic cost \citep{makkuva2020optimal,amos2022amortizing,finlay2020learning} or rely heavily on knowing the exact cost function \citep{fan2021scalable,asadulaev2022neural}. Being capable of using batch optimal couplings allows us to build generative models to approximate optimal maps under any cost function, and even when the cost function is unknown.

\begin{figure}[t]
    \centering
    \begin{subfigure}[b]{0.49\linewidth}
        \includegraphics[width=\linewidth]{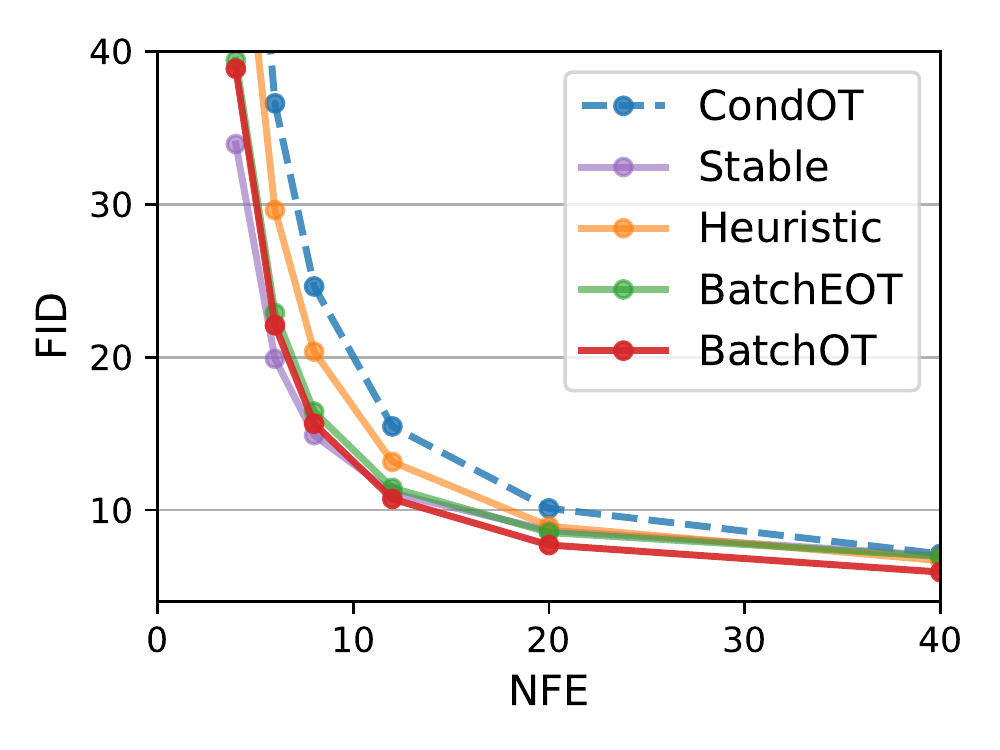}
        \caption*{ImageNet 32$\times$32}
    \end{subfigure}
    \begin{subfigure}[b]{0.49\linewidth}
        \includegraphics[width=\linewidth]{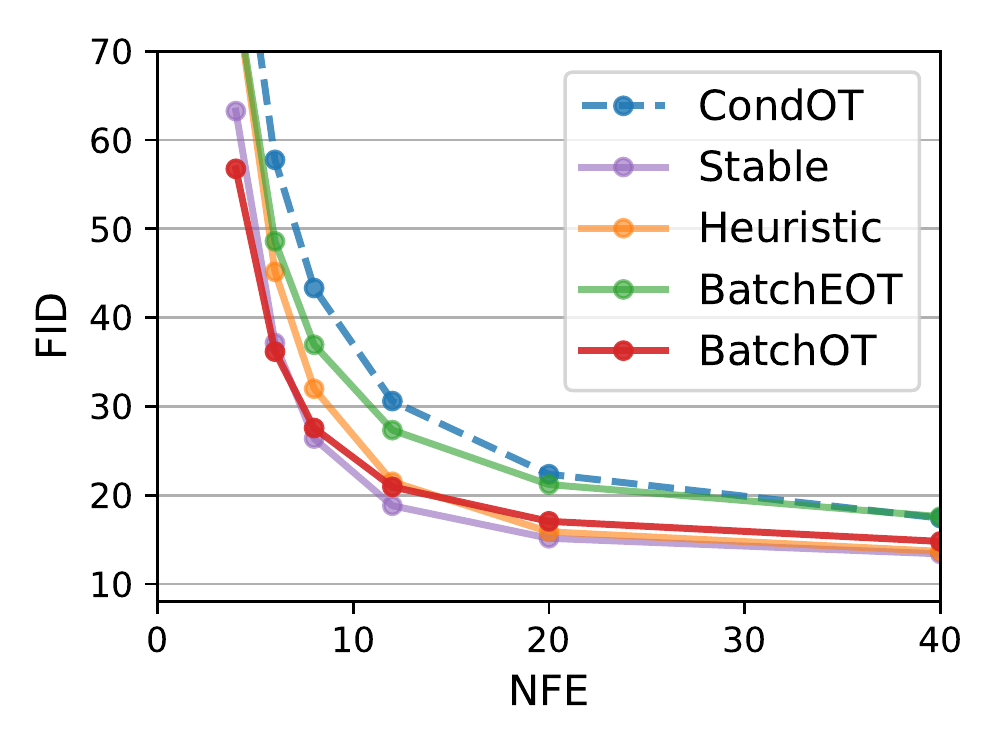}
        \caption*{ImageNet 64$\times$64}
    \end{subfigure}
    \caption{Sample quality (FID) vs compute cost (NFE) using Euler discretization. CondOT has significantly higher FID at lower NFE compared to proposed methods.}
    \label{fig:nfefid_euler}
\end{figure}

\subsection{Minibatch Couplings for Generative Modeling}\label{sec:related_batchot}
Among works that use optimal transport for training generative models are those that make use of batch optimal solutions and their gradients such as \citet{li2017mmd,genevay2018learning,fatras2019learning,liu2019wasserstein}. However, \textit{na\"ively using solutions to batches only produces, at best, the barycentric map}, \ie the map that fits to average of the batch couplings \citep{ferradans2014regularized,seguy2017large,pooladian2021entropic}, and does not correctly match the true marginal distribution.
This is a well-known problem and while multiple works (\eg \citet{fatras2021minibatch,nguyen2022improving}) have attempted to circumvent the issue through alternative formulations of optimality, the lack of marginal preservation has been a major downside of using batch couplings for generative modeling as they do not have the ability to match the target distribution for finite batch sizes. 
This is due to the use of building models within the \emph{static} setting, where the map is parameterized directly with a neural network. 
In contrast, we have shown in \cref{lem:q_marginals} that in our \emph{dynamic} setting, where we parameterize the map as the solution of a neural ODE, it is possible to preserve the marginal distribution exactly. Furthermore, we have shown in Proposition~\ref{thm:transport_cost} (App.~\ref{subsec:monotone}) that our method produces a map that is no higher cost than the joint distribution induced from BatchOT couplings.

Concurrently, \citet{tong2023conditional} motivates the use of BatchOT solutions within a similar framework as our Joint CFM, but from the perspective of obtaining accurate solutions to dynamic optimal transport problems. Similarly, \citet{lee2023minimizing} propose to explicitly learn a joint distribution, parameterized with a neural network, with the aim of minimizing trajectory curvature; this is done using through an auxiliary VAE-style objective function. 
In contrast, we propose a family of couplings that all satisfy the marginal constraints, all of which are easy to implement and have negligible cost during training.
Our construction allow us to focus on (i) fixing consistency issues within simulation-free generative models, and (ii) using Joint CFM to obtain more optimal solutions than the original BatchOT solutions.
\looseness=-1

\section{Experiments}

\begin{table}[t]
\centering
\ra{1.1}
\resizebox{0.95\columnwidth}{!}{%
\begin{tabular}{ l r r }\toprule
  & \multicolumn{1}{c}{\bf ImageNet 32$\times$32}
  & \multicolumn{1}{c}{\bf ImageNet 64$\times$64} \\
& {NFE @ FID $= 10$}  
& {NFE @ FID $= 20$} \\
\midrule
Diffusion & 
$\geq$40 &
$\geq$40 \\
FM \textsuperscript{w}/ CondOT & 
20 & 
29 \\
MultisampleFM \textsuperscript{w}/ Heuristic &
18 & 
12 \\
MultisampleFM \textsuperscript{w}/ Stable & 
\textbf{14} &  
\textbf{11}  \\
MultisampleFM \textsuperscript{w}/ BatchOT  & 
\textbf{14} &  
12 \\
\bottomrule
\end{tabular}
}
\caption{ Derived results shown in \cref{fig:nfefid_euler}, we can determine the approximate NFE required to achieve a certain FID across our proposed methods. The baseline diffusion-based methods (e.g. ScoreFlow and DDPM) require more than 40 NFE to achieve these FID values.}
\label{tab: NFE_at_FID_table}
\end{table}

\begin{table}[t]
\centering
\ra{1.0}
\resizebox{0.9\columnwidth}{!}{%
\begin{tabular}{ c r r r r } \toprule
\multicolumn{1}{l}{ NFE} & \multicolumn{1}{r}{ DDPM} & \multicolumn{1}{r}{ ScoreSDE} & \multicolumn{1}{r}{ BatchOT} & \multicolumn{1}{r}{ Stable} \\ \midrule
\multicolumn{1}{l}{ Adaptive \hspace{1em}}  & 5.72 & 6.84 & \textbf{4.68} & 5.79 \\
\multicolumn{1}{l}{ 40}  & 19.56  & 16.96  & \textbf{5.94}  & 7.02  \\
\multicolumn{1}{l}{ 20}  & 63.08  & 58.02  & \textbf{7.71}  & 8.66  \\
\multicolumn{1}{l}{ 8}   & 232.97 & 218.66 & 15.64 & \textbf{14.89} \\
\multicolumn{1}{l}{ 6}   & 275.28 & 266.76 & 22.08 & \textbf{19.88} \\
\multicolumn{1}{l}{ 4}   & 362.37 & 340.17 & 38.86 & \textbf{33.92} \\
\bottomrule
\end{tabular}
}
\caption{FID of model samples on ImageNet 32$\times$32 using varying number of function evaluations (NFE) using Euler discretization. }
\label{tab:NFE_at_FID_table_euler_32}
\end{table}

\begin{table}[t]
\centering
\ra{1.1}
\setlength{\tabcolsep}{2.0pt}
\resizebox{0.9\columnwidth}{!}{%
\begin{tabular}{ l R{4em} R{4em} R{4em} R{4em} }\toprule
  & \multicolumn{2}{c}{\bf ImageNet 32$\times$32}
  & \multicolumn{2}{c}{\bf ImageNet 64$\times$64} \\
& {CondOT} & {BatchOT}  
& {CondOT} & {BatchOT} \\
\midrule
Consistency($m$=4) & 
0.141 & \textbf{0.101}  & 
0.174 & \textbf{0.157}  \\
Consistency($m$=6) & 
0.105 & \textbf{0.071} & 
0.151 & \textbf{0.134} \\
Consistency($m$=8) & 
0.079 & \textbf{0.052} & 
0.132 & \textbf{0.115} \\
Consistency($m$=12) & 
0.046 & \textbf{0.030} & 
0.106 & \textbf{0.08}5 \\
\bottomrule
\end{tabular}
}
\caption{BatchOT produces samples with more similar content to its true samples at low NFEs (using midpoint discretization). Visual examples of this consistency are shown in \cref{fig:samples_vs_nfe_imagenet64}.}
\label{tab:Consistency_table}
\end{table}

We empirically investigate Multisample Flow Matching on a suite of experiments. First, we show how different couplings affect the model on a 2D distribution. We then turn to benchmark, high-dimensional datasets, namely ImageNet \citep{deng2009imagenet}. We use the official \emph{face-blurred} ImageNet data and then downsample to 32$\times$32 and 64$\times$64 using the open source preprocessing scripts from \citet{chrabaszcz2017downsampled}. Finally, we explore the setting of unknown cost functions while only batch couplings are provided. Full details on the experimental setting can be found in Appendix \ref{app:exp_stat_vs_dyn}.\looseness=-1

\begin{table*}
\centering
\ra{1.3}
\setlength{\tabcolsep}{3.0pt}
\resizebox{1.0\linewidth}{!}{%
\begin{tabular}{@{} l @{\hspace{5mm}} ccc l@{\hspace{3mm}} cc l@{\hspace{5mm}} ccc l@{\hspace{3mm}} cc l@{\hspace{5mm}} ccc l@{\hspace{3mm}} cc @{}} 
\toprule
 & \multicolumn{3}{c}{2-D Cost} 
 & & \multicolumn{2}{c}{2-D KL}
 & & \multicolumn{3}{c}{32-D Cost} 
 & & \multicolumn{2}{c}{32-D KL} 
 & & \multicolumn{3}{c}{64-D Cost}
 & & \multicolumn{2}{c}{64-D KL} \\ 
\cmidrule(r){2-4} 
\cmidrule{6-7}
\cmidrule(r){9-11}
\cmidrule{13-14}
\cmidrule(r){16-18}
\cmidrule{20-21}
Cost Fn. $c(x_0,x_1)$
 & B & B-\textsc{st} & B-\textsc{fm}
 & & B-\textsc{st} & B-\textsc{fm}
 & & B & B-\textsc{st} & B-\textsc{fm}
 & & B-\textsc{st} & B-\textsc{fm}
 & & B & B-\textsc{st} & B-\textsc{fm}
 & & B-\textsc{st} & B-\textsc{fm} \\
 \midrule
$\norm{x_1-x_0}_2^2$ 
& 0.90 & 0.60 & 0.72 & & 0.07 & 4E-3 &
& 41.08 & 31.58 & 38.73 & & 151.47 & 0.06 &
& 92.90 & 65.57 & 87.97 & & 335.38 & 0.14 \\
$\norm{x_1-x_0}_1$ 
& 1.09 & 0.86 & 0.98 & & 0.18 & 4E-3 &
& 27.92 & 24.51 & 27.26 & & 254.59 & 0.08 &
& 60.27 & 50.49 & 58.38 & & 361.16 & 0.16 \\
$1-\frac{\langle x_0,x_1\rangle}{\norm{x_0}\norm{x_1}}$ 
& 0.03 & 2E-4 & 3E-3 & & 5.91 & 4E-3 &
& 0.62 & 0.53 & 0.58 & & 179.48 & 0.06 &
& 0.71 & 0.60 & 0.68 & & 337.63 & 0.12 \\
$\norm{A(x_1-x_0)}_2^2$ 
& 0.91 & 0.54 & 0.65 & & 0.07 & 4E-3 &
& 32.66 & 24.61 & 30.13 & & 256.90 & 0.06 &
& 78.70 & 58.11 & 78.50 & & 529.09 & 0.19 \\
\bottomrule
\end{tabular}
}
\caption{Matching couplings from an oracle BatchOT solver with unknown costs. Multisample Flow Matching is able to match the marginal distribution correctly while being at least a optimal as the oracle, but static maps fail to preserve the marginal distribution.}
\label{tab:t}
\end{table*}

\begin{figure}[t]
    \centering
    \includegraphics[width=0.75\linewidth]{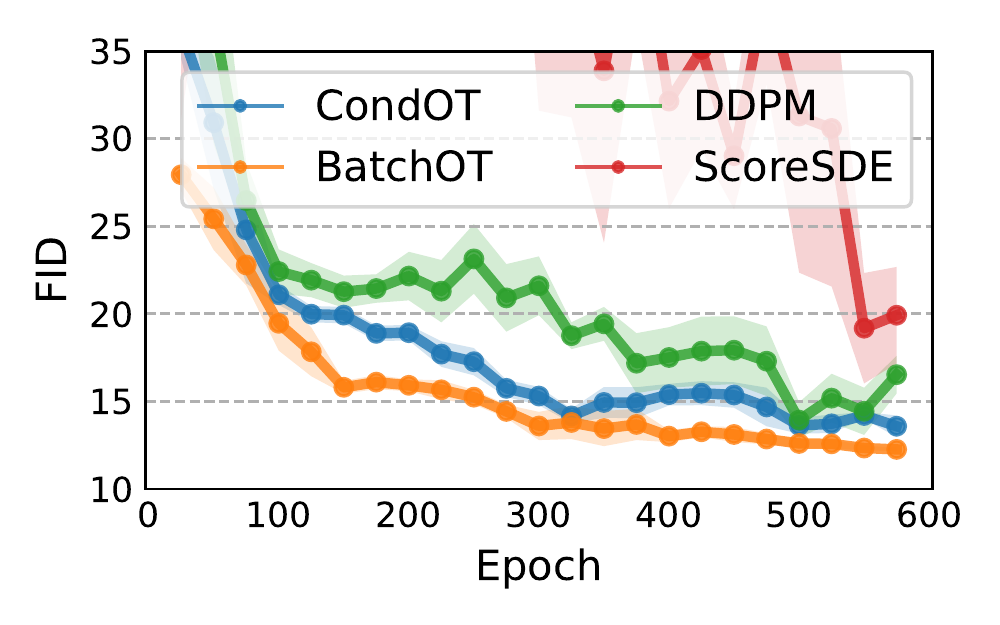}
    \vspace{-1em}
    \caption{Multisample Flow Matching with BatchOT shows faster convergence due to reduced variance (ImageNet64).}
    \label{fig:training_efficiency}
\end{figure}

\subsection{Insights from 2D experiments}
\cref{fig:2D_1} shows the proposed Multisample Flow Matching algorithm on fitting to a checkboard pattern distribution in 2D. We show the marginal probability paths induced by different coupling algorithms, as well as low-NFE samples of trained models on these probability paths.

The diffusion and CondOT probability paths do not capture intricate details of the data distribution until it is almost at the end of the trajectory, whereas Multisample Flow Matching approaches provide a gradual transition to the target distribution along the flow. 
We also see that with a fixed step solver, the BatchOT method is able to produce an accurate target distribution in just one Euler step in this low-dimensional setting, while the other coupling approaches also get pretty close.
Finally, it is interesting that both Stable and Heuristic exhibit very similar probability paths to optimal transport despite only satisfying weaker conditions.

\subsection{Image Datasets}

We find that Multisample Flow Matching retains the performance of Flow Matching while improving on sample quality, compute cost, and variance. 
In \cref{tab:img_results} of \cref{sec:full_imagenet_results}, we report sample quality using the standard Fr{\'e}chet Inception Distance (FID), negative log-likelihood values using bits per dimension (BPD), and compute cost using number of function evaluations (NFE); these are all standard metrics throughout the literature. Additionally, we report the variance of $u_t(x|x_0, x_1)$, estimated using the Joint CFM loss \eqref{eq:cfm_joint} which is an upper bound on the variance. We do not observe any performance degradations while simulation efficiency improves significantly, even with small batch sizes.\looseness=-1

Additionally, in \cref{app:runtime_comparison}, we include runtime comparisons between Flow Matching and Multisample Flow Matching. On ImageNet32, we only observe a 0.8\% relative increase in runtime compared to Flow Matching, and a 4\% increase on ImageNet64.

\vspace{-0.5em}
\paragraph*{Higher sample quality on a compute budget} We observe that with a fixed NFE, models trained using Multisample Flow Matching generally achieve better sample quality. 
For these experiments, we draw $x_0 \sim \mathcal{N}(0,I_d)$ and simulate $v_t(\cdot,\theta)$ up to time $t=1$ using a fixed step solver with a fixed NFE.
Figures \ref{fig:nfefid_euler} show that even on high dimensional data distributions, the sample quality of of multisample methods improves over the na{\"i}ve CondOT approach as the number of function evaluations drops. We compare to the FID of diffusion baseline methods in \cref{tab:NFE_at_FID_table_euler_32}, and provide additional results in \cref{app:FID_vs_NFE}.

Interestingly, we find that the Stable coupling actually performs on par, and some times better than the BatchOT coupling, despite having a smaller asymptotic compute cost and only satisfying a weaker condition within each batch.

As FID is computed over a full set of samples, it does not show how varying NFE affects individual sample paths. We discuss a notion of consistency next, where we analyze the similarity between low-NFE and high-NFE samples.

\vspace{-0.5em}
\paragraph*{Consistency of individual samples}
In Figure \ref{fig:samples_vs_nfe_imagenet64} we show samples at different NFEs, where it can be qualitatively seen that BatchOT produces samples that are more consistent between high- and low-NFE solutions than CondOT, despite achieving similar FID values.

To evaluate this quantitatively, we define a metric for establishing the \textit{consistency} of a model with respect to an integration scheme: let $x^{(m)}$ be the output of a numerical solver initialized at $x$ using $m$ function evalutions to reach $t=1$, and let $x^{(*)}$ be a near-exact sample solved using a high-cost solver starting from $x_0$ as well. We define
\begin{align}\label{eq:consistency_metric}
    \!\!\!\!\! \text{Consistency}(m) = \tfrac{1}{D} \E_{x \sim q_0}\| \mathcal{F}(x^{(m)}) - \mathcal{F}(x^{(*)})\|^2
\end{align}
where $\mathcal{F}(\cdot)$ outputs the hidden units from a pretrained InceptionNet\footnote{We take the same layer as used in standard FID computation.}, and $D$ is the number of hidden units. These kinds of perceptual losses have been used before to check the content alignment between two image samples (\eg \citet{gatys2015neural,johnson2016perceptual}). We find that Multisample Flow Matching has better consistency at all values of NFE, shown in \cref{tab:Consistency_table}.

\paragraph{Training efficiency} \Cref{fig:training_efficiency} shows the convergence of Multisample Flow Matching with BatchOT coupling compared to Flow Matching with CondOT and diffusion-based methods. We see that by choosing better joint distributions, we obtain faster training. This is in line with our variance estimates reported in \cref{tab:img_results} and supports our hypothesis that gradient variance is reduced by using non-trivial joint distributions.\looseness=-1

\subsection{Improved Batch Optimal Couplings}
\label{sec:improve_batchot}
We further explore the usage of Multisample Flow Matching as an approach to improve upon batch optimal solutions.
Here, we experiment with a different setting, where the cost is unknown and only samples from a batch optimal coupling are provided. In the real world, it is often the case that the preferences of each person are not known explicitly, but when given a finite number of choices, people can more easily find their best assignments. This motivates us to consider the case of unknown cost functions, and information regarding the optimal coupling is only given by a weak oracle that acts on finite samples, denoted $q_{OT,c}^k$. We consider two baselines: (i) the BatchOT cost (B) which corresponds to $\E_{q_{OT,c}^k(x_0, x_1)}\brac{c(x_0, x_1)}$, and (ii) learning a static map that mimics the BatchOT couplings (B-ST) by minimizing the following objective:
\begin{equation}
    \E_{ q_{OT,c}^k(x_0,x_1) }\norm{x_1 - \psi_\theta(x_0)}^2\,.
\end{equation}
This can be viewed as learning the barycentric projection \citep{ferradans2014regularized,seguy2017large}, \ie $\psi^*(x_0) = E_{q_{OT,c}^k(x_1 | x_0)} \left[x_1\right]$, a well-studied quantity but is known to not preserve the marginal distribution \citep{fatras2019learning}. 

\begin{figure}[t]
    \centering
        \includegraphics[width=\linewidth]{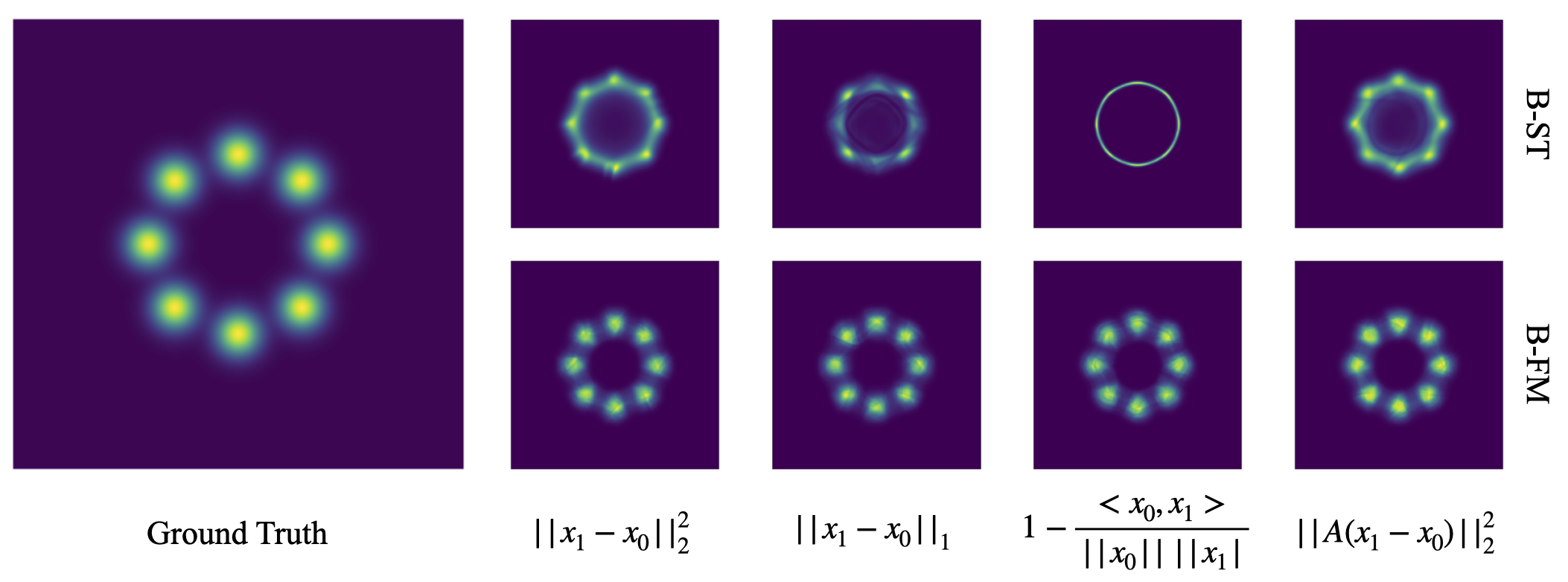}
        \vspace{-1.5em}
    \caption{2D densities on the 8-Gaussians target distribution. (Left) Ground truth density. (Right) Learned densities with static maps in the top row and Multisample Flow Matching dynamic maps in the bottom row. Models within each column were trained using batch optimal couplings with the corresponding cost function.}
    \label{fig:nfefid_midpoint}
\end{figure}

\begin{figure}[t]
    \centering
        \includegraphics[width=0.7\linewidth]{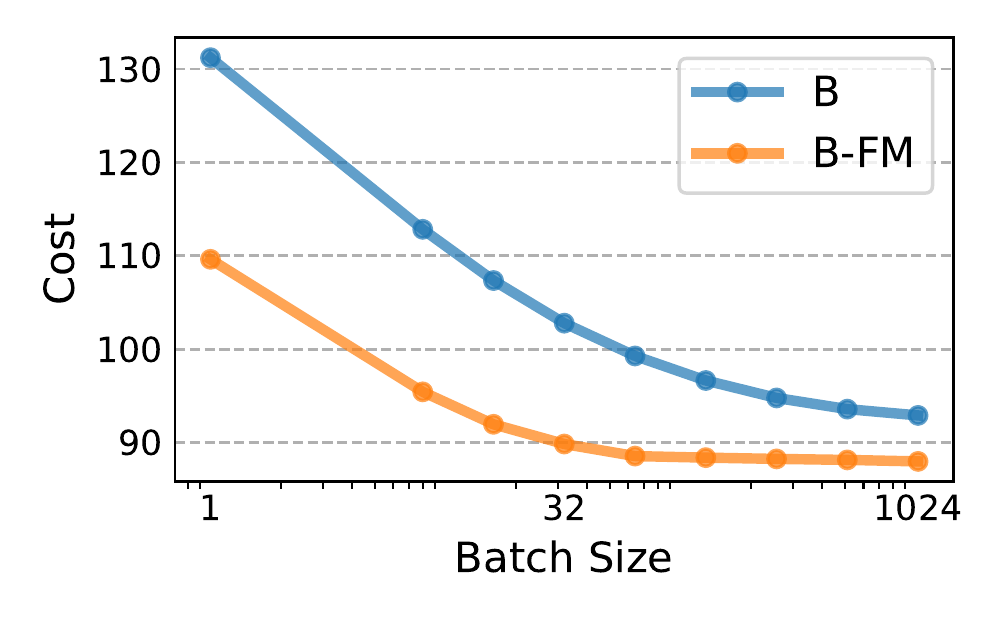}
        \vspace{-1.5em}
    \caption{Transport cost vs. batch size ($k$) for computing couplings on the 64D synthetic dataset. The number of samples used for performing gradient steps during training and the resulting KL divergences were kept the same.}
    \label{fig:cost_vs_bs}
\end{figure}

We experiment with 4 different cost functions on three synthetic datasets in dimensions $\{2,32,64\}$ where both $q_0$ and $q_1$ are chosen to be Gaussian mixture models. 
In Table \ref{tab:t} we report both the transport cost and the KL divergence between $q_1$ and the distribution induced by the learned map, \ie $[\psi_1]_{\sharp}q_0$. 
We observe that while B-ST always results in lower transport costs compared to B-FM, its KL divergence is always very high, meaning that the pushed-forward distribution by the learned static map poorly approximates $q_1$. 
Another interesting observation is that B-FM always reduces transport costs compared to B, providing experimental support to the theory (\cref{thm:monotone}). 
\paragraph{Flow Matching improves optimality} \cref{fig:cost_vs_bs} shows the cost of the learned model as we vary the batch size for computing couplings, where the models  are trained sufficiently to achieve the same KL values as reported in \cref{tab:t}. We see that our approach decreases the cost compared to the BatchOT oracle for any fixed batch size, and furthermore, converges to the OT solution faster than the batchOT oracle. Thus, since Multisample Flow Matching retains the correct marginal distributions, it can be used to better approximate optimal transport solutions than simply relying on a minibatch solution.\looseness=-1

\section{Conclusion}

We propose Multisample Flow Matching, building on top of recent works on simulation-free training of continuous normalizing flows. 
While most prior works make use of training algorithms where data and noise samples are sampled independently, Multisample Flow Matching allows the use of more complex joint distribution. 
This introduces a new approach to designing probability paths.
Our framework increases sample efficiency and sample quality when using low-cost solvers. 
Unlike prior works, our training method does not rely on simulation of the learned vector field during training, and does not introduce any min-max formulations.
Finally, we note that our method of fitting to batch optimal couplings is the first to also preserve the marginal distributions, an important property in both generative modeling and solving transport problems.

\ifdefined\isaccepted 
\subsection*{Acknowledgements}

AAP thanks the Meta AI Mentorship program and NSF Award 1922658 as funding sources. HB  was supported by a grant from Israel CHE Program for
Data Science Research Centers.
Additionally, we acknowledge the Python community
\citep{van1995python,oliphant2007python}
for developing
the core set of tools that enabled this work, including
PyTorch \citep{paszke2019pytorch},
PyTorch Lightning \citep{lightning},
Hydra \citep{Yadan2019Hydra},
Jupyter \citep{kluyver2016jupyter},
Matplotlib \citep{hunter2007matplotlib},
seaborn \citep{seaborn},
numpy \citep{oliphant2006guide,van2011numpy},
pandas \citep{mckinney2012python}, 
SciPy \citep{jones2014scipy}, 
pot \citep{flamary2021pot}, and
torchdiffeq \citep{torchdiffeq}.
\fi

\bibliography{biblio}
\bibliographystyle{icml2023}

\newpage
\appendix
\onecolumn

\section{Coupling algorithms} \label{app:couplings}

Multisample FM makes use of batch coupling algorithms to construct an implicit joint distribution satisfying the marginal constraints. While BatchOT coupling is motivated by approximating the OT map, we consider other lower complexity coupling algorithms which produce coupling that satisfy some desired property of optimal couplings. In Table \ref{tab:runtimes} we summarize the runtime complexities for the different algorithms used in this work. We will now describe in detail the Stable and Heuristic coupling algorithms.

\begin{table}[h]
\centering
\ra{1.3}
\resizebox{0.7\columnwidth}{!}{%
 \small
\begin{tabular}{cccccc} 
\toprule
 & CondOT & BatchOT & BatchEOT & Stable & Heuristic \\ 
\hline
Runtime Complexity & $\bigO{(1)}$ & $\bigO{(k^3)}$ & $\tilde{\mathcal{O}}(k^2/\eps)$ & $\bigO{(k^2\log(k))}$ & $\bigO{(k^2\log(k))}$ \\
\bottomrule
\end{tabular}
}
\caption{Runtime complexities of the different coupling algorithms as a function of the batch size $k$.}
\label{tab:runtimes}
\end{table}

\subsection{Stable couplings }
\cite{Wolansky_2020} surveys discrete optimal transport from a stable coupling perspective proving that stability is a necessary condition for OT couplings. Although stable couplings are not OT, they are cheaper to compute and are therefore an appealing approach to pursue. For completeness we formulate the Gale Shapely Algorithm in our setting in Algorithm 1. The rankings $R_0,R_1$ hold the preferences of the samples in $\{x_0^{(i)}\}_{i=1}^k$ and $\{x_1^{(i)}\}_{i=1}^k$ respectively. Where $R_0(i,j)$ is the rank of $x_1^{(j)}$ in $x_0^{(i)}$'s preferences and $R_1(i,j)$ is the rank of $x_0^{(j)}$ in $x_1^{(i)}$'s preferences.

\subsection{Heuristic couplings } \label{app:heuristic_couplings}
The stable coupling is agnostic to the cost of pairing samples and only takes into account the ranks. Therefore, reassignments during the Gale Shapely algorithms might increase the total cost although the rankings of assigned samples are improved. We draw inspiration from the cyclic monotonicity of OT couplings \cite{villani2008optimal} and from the marriage with sharing formulation in \cite{Wolansky_2020} and modify the reassignment condition in the Gale Shapely algorithm (see Algorithm 2). The modified condition encourages "local" monotonicity between the reassigned pairs only, reassigning a pair only if the potentially newly assigned pairs have a lower cost.

\begin{minipage}{0.46\textwidth}
\begin{algorithm}[H]
\label{algo:stable}
\SetAlgoLined
\KwResult{assignment $\sigma$ }
\KwData{$\smash{\{x_0^{(i)}\}_{i=1}^k \sim q_0(x_0)}$, $\smash{\{x_1^{(i)}\}_{i=1}^k \sim q_1(x_1)}$, rankings $R_0,R_1$}
 initialization:\; $\sigma$ empty assignment \\ 
 \While{$\exists \;i\in[k]$ \texttt{s.t.} $\sigma(i)$ is empty }{
  $j \gets$ first sample in $R_0(i,\cdot)$  whom $x_0^{(i)}$ has not tried to match with yet \\ 
  \eIf{$\exists \; i'$ \texttt{s.t.} $\sigma(i')=j$}{
       \If{$R_1(j,i)<R_1(j,i')$}{
           $\sigma(i') \gets $ empty \\
           $\sigma(i) \gets j$
       }
   }
   {
   $\sigma(i) \gets j$
  }
 }
 \caption{Stable Coupling (Gale Shapely)}
 \vspace{54pt}
\end{algorithm}
\end{minipage}
\hfill
\begin{minipage}{0.46\textwidth}
\begin{algorithm}[H]
\SetAlgoLined
\KwResult{assignment $\sigma$ }
\KwData{$\smash{\{x_0^{(i)}\}_{i=1}^k \sim q_0(x_0)}$, $\smash{\{x_1^{(i)}\}_{i=1}^k \sim q_1(x_1)}$, rankings $R_0,R_1$, cost matrix $C$}
 initialization:\; $\sigma$ empty assignment \\ 
 \While{$\exists \;i\in[k]$ \texttt{s.t.} $\sigma(i)$ is empty }{
  $j \gets$  first sample in $R_0(i,\cdot)$  whom $x_0^{(i)}$ has not tried to match with yet  \\ 
  \eIf{$\exists \; i'$ \texttt{s.t.} $\sigma(i')=j$}{
        $j' \gets$  first sample in $R_0(i',\cdot)$  whom $x_0^{(i')}$ has not tried to match with yet   \\ 
        $l \gets$ second sample in $R_0(i,\cdot)$  whom $x_0^{(i)}$ has not tried to match with yet  \\
       \If{$C(i,j) + C(i',j') < C(i,l) + C(i',j)$}{
           $\sigma(i') \gets $ empty \\
           $\sigma(i) \gets j$
       }
   }
   {
   $\sigma(i) \gets j$
  }
 }
 \caption{Heuristic Coupling}
\end{algorithm}
\end{minipage}

\pagebreak
\section{Additional tables and figures} \label{sec:additional_figs}
\subsection{Full results on ImageNet data}\label{sec:full_imagenet_results}

\begin{table}[H]
\centering
\ra{1.05}
\resizebox{1.0\linewidth}{!}{%
\begin{tabular}{@{} l  r r r r  r  r r r r  r @{}}\toprule
 & \multicolumn{4}{c}{\bf ImageNet 32$\times$32} &
  & \multicolumn{4}{c}{\bf ImageNet 64$\times$64} \\
\cmidrule(lr){2-6} \cmidrule(l){7-10}
Model 
& {NLL} & {FID} & {NFE} & {Var($u_t$)}
& & {NLL} & {FID} & {NFE} & {Var($u_t$)} \\
\cmidrule(r){1-1}\cmidrule(lr){2-6} \cmidrule(l){7-10}
\textit{\footnotesize Ablations$^\dagger$}\\
\;\; DDPM~{\scriptsize \citep{ho2020denoising}} & 
3.61 & 5.72 & 330 &  & &
3.27 & 13.80 & 323 &  \\
\;\; ScoreSDE~{\scriptsize \citep{song2020score}} & 
3.61 & 6.84 & 198 &  & &
3.30 & 26.64 & 365 &  \\
\;\; ScoreFlow~{\scriptsize \citep{song2021maximum}}  & 
3.61 & 9.53 & 189 &  & &
3.34 & 32.78 & 554 &  \\
\;\; Flow Matching \textsuperscript{w}/ {\footnotesize Diffusion}~{\scriptsize \citep{lipman2022flow}} & 
3.60 & 6.36 & 165 &  & &
3.35 & 15.11 & 162 &  \\
\;\; Rectified Flow~{\scriptsize \citep{liu2022flow}} & 
3.59 & 5.55 & 111 &  & &
3.31 & 13.02 & 129 &  \\
\;\; Flow Matching \textsuperscript{w}/ {\footnotesize CondOT}~{\scriptsize \citep{lipman2022flow}} & 
3.58 & 5.04 & 139 & 594 & &
3.27 & 13.93 & 131 & 1880 \\
\cmidrule(r){1-1}\cmidrule(lr){2-6} \cmidrule(l){7-10}
\textit{\footnotesize Ours}\\
\;\; Multisample Flow Matching \textsuperscript{w}/ {\footnotesize StableCoupling} &  
3.59 & 5.79 & 148 & 523 & &
3.27 & 11.82 & 132 & 1782 \\
\;\; Multisample Flow Matching \textsuperscript{w}/ {\footnotesize HeuristicCoupling} &  
3.58 & 5.29 & 133 & 555 & &
3.26 & 13.37 & 110 & 1816 \\
\;\; Multisample Flow Matching \textsuperscript{w}/ {\footnotesize BatchEOT} &  
3.58 & 6.14 & 132 & 508 & &
3.26 & 14.92 & 141 & 1736 \\
\;\; Multisample Flow Matching \textsuperscript{w}/ {\footnotesize BatchOT} &  
3.58 & 4.68 & 146 & 507 & &
3.27 & 12.37 & 135 & 1733  \\
\bottomrule
\end{tabular}
}
\caption{Multisample Flow Matching improves on sample quality and sample efficiency while not trading off performance at all compared to Flow Matching. $^\dagger$Reproduction using the same training hyperparameters (architecture, optimizer, training iterations) as our methods.}
\label{tab:img_results}
\end{table}

\subsection{How batch size affects the marginal probability paths on 2D checkerboard data}

\begin{figure}[H]
    \centering
    \rotatebox{90}{\hspace{0.5em}\small BatchOT}
    \includegraphics[width=0.98\textwidth]{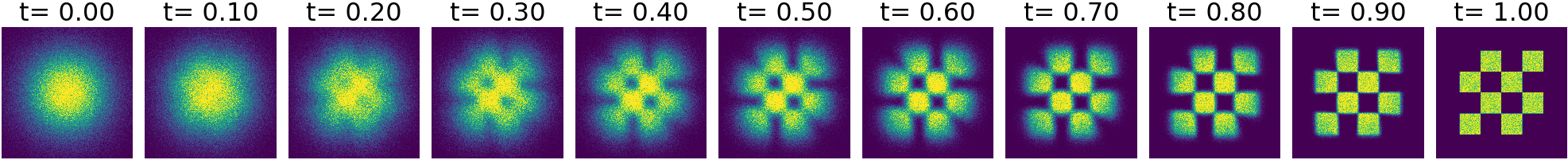} \\
    \rotatebox{90}{\hspace{0.7em}\small Stable}
    \includegraphics[width=0.98\textwidth]{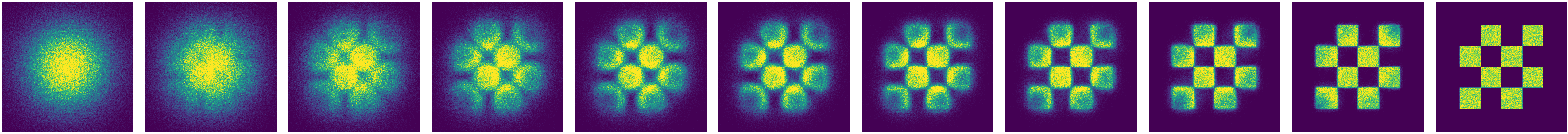} \\
    \rotatebox{90}{\hspace{0.5em}\small Heuristic}
    \includegraphics[width=0.98\textwidth]{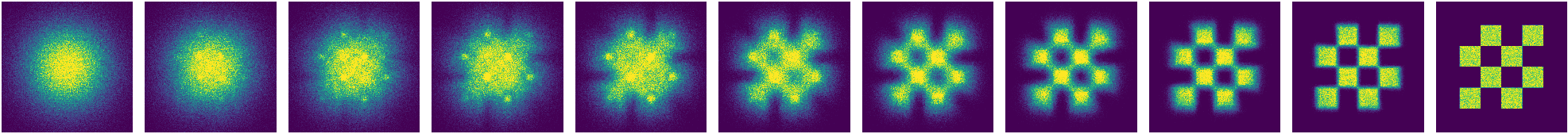}
    \vspace{0pt}
    \hrule
    \vspace{6pt}
    \rotatebox{90}{\hspace{0.5em}\small BatchOT}
    \includegraphics[width=0.98\textwidth]{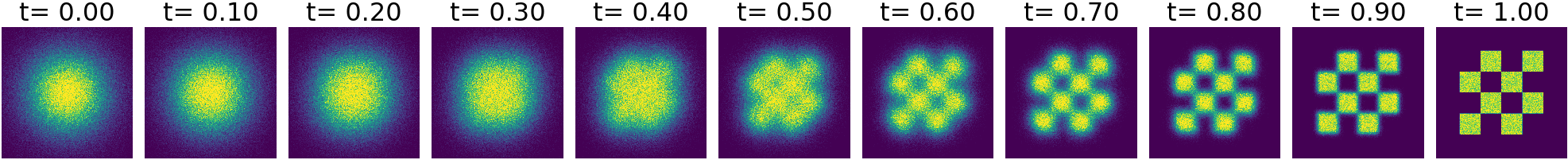} \\
    \rotatebox{90}{\hspace{0.7em}\small Stable}
    \includegraphics[width=0.98\textwidth]{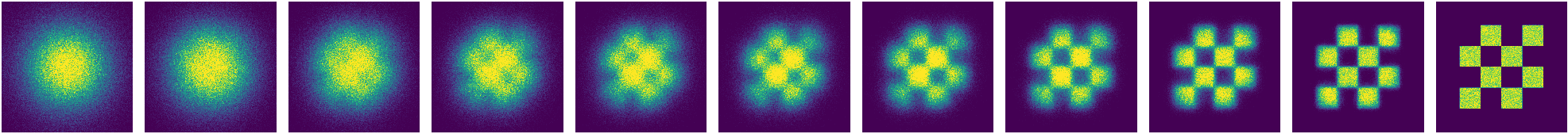} \\
    \rotatebox{90}{\hspace{0.5em}\small Heuristic}
    \includegraphics[width=0.98\textwidth]{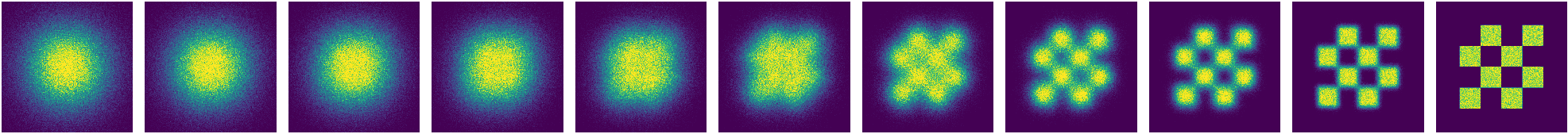}
    \caption{Marginal probability paths. (\emph{Top}) Batch size 64. (\emph{Bottom}) Batch size 8.}
    \label{fig:2D_3}
\end{figure}

\subsection{FID vs NFE using midpoint discretization scheme}

\begin{figure}[H]
    \centering
    \begin{subfigure}[b]{0.4\linewidth}
        \includegraphics[width=\linewidth]{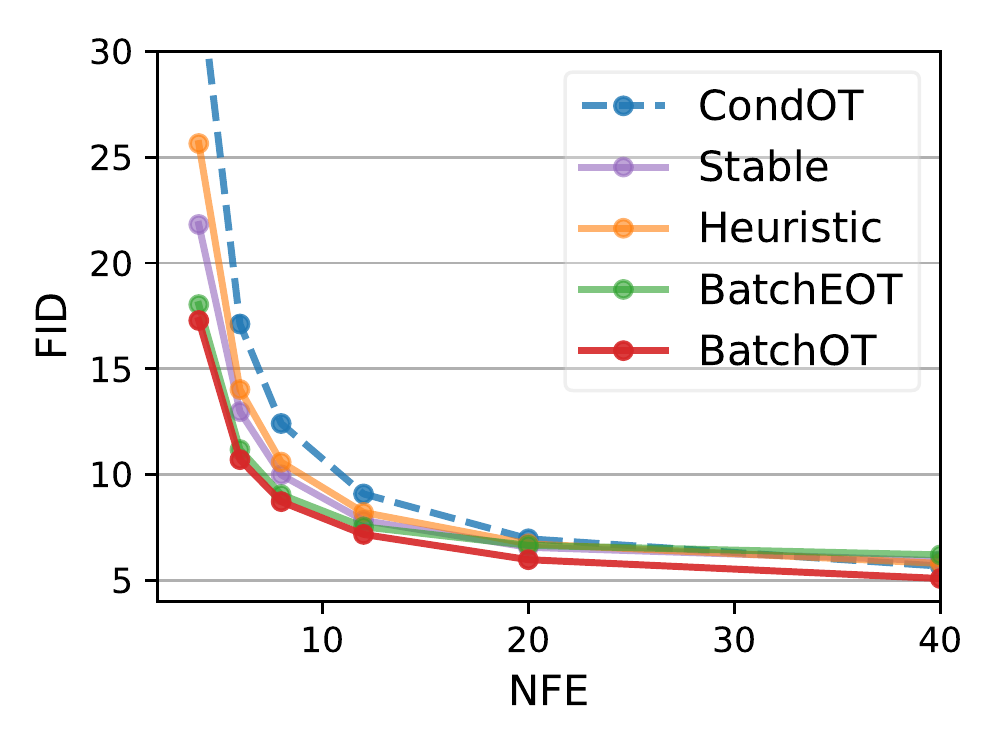}
        \vspace{-1.5em}
        \caption*{ImageNet 32$\times$32}
    \end{subfigure}
    \begin{subfigure}[b]{0.4\linewidth}
        \includegraphics[width=\linewidth]{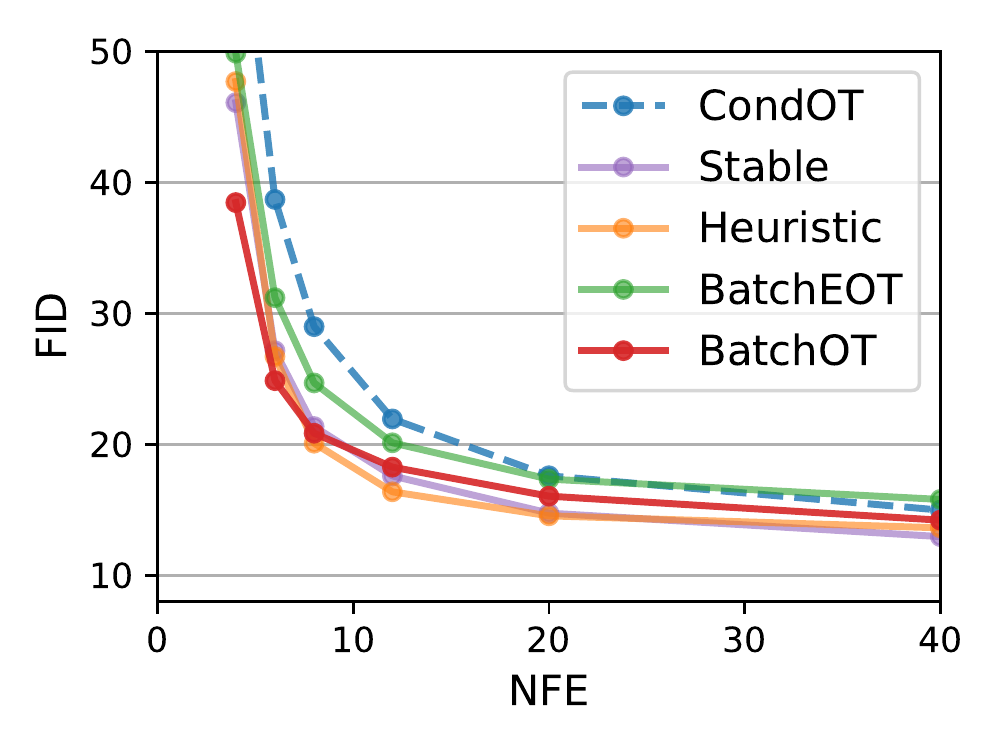}
        \vspace{-1.5em}
        \caption*{ImageNet 64$\times$64}
    \end{subfigure}
    \caption{Sample quality (FID) vs compute cost (NFE); midpoint discretization.}
    \label{fig:nfefid_midpoint}
\end{figure}

\subsection{Comparison of FID vs NFE for baseline methods DDPM and ScoreSDE}\label{app:FID_vs_NFE}

\begin{table}[H]
\centering
\begin{minipage}{.5\linewidth}

\begin{tabular}{ c r r r r } 
\multicolumn{5}{c}{ ImageNet32 FID (Euler) }  \\ 
\toprule
\multicolumn{1}{l}{ NFE} & \multicolumn{1}{r}{ DDPM} & \multicolumn{1}{r}{ ScoreSDE} & \multicolumn{1}{r}{ BatchOT} & \multicolumn{1}{r}{ Stable} \\ \midrule
\multicolumn{1}{l}{ Adaptive \hspace{1em}}  & 5.72 & 6.84 & \textbf{4.68} & 5.79 \\
\multicolumn{1}{l}{ 40}  & 19.56  & 16.96  & \textbf{5.94}  & 7.02  \\
\multicolumn{1}{l}{ 20}  & 63.08  & 58.02  & \textbf{7.71}  & 8.66  \\
\multicolumn{1}{l}{ 12}  & 152.59 & 140.95 & \textbf{10.72} & 11.10 \\
\multicolumn{1}{l}{ 8}   & 232.97 & 218.66 & 15.64 & \textbf{14.89} \\
\multicolumn{1}{l}{ 6}   & 275.28 & 266.76 & 22.08 & \textbf{19.88} \\
\multicolumn{1}{l}{ 4}   & 362.37 & 340.17 & 38.86 & \textbf{33.92} \\
\bottomrule
\end{tabular}

\end{minipage}%
\begin{minipage}{.5\linewidth}
\begin{tabular}{ c r r r r }
\multicolumn{5}{c}{ ImageNet32 FID (Midpoint) }  \\ 
\toprule
\multicolumn{1}{l}{ NFE} & DDPM & ScoreSDE & BatchOT & Stable \\ 
\midrule
\multicolumn{1}{l}{ Adaptive \hspace{1em}}  & 5.72 & 6.84 & \textbf{4.68} & 5.79 \\
\multicolumn{1}{l}{ 40}  & 6.68 & 6.48 & \textbf{5.09} & 5.94 \\
\multicolumn{1}{l}{ 20}  & 7.80 & 8.96 & \textbf{5.98} & 6.57 \\
\multicolumn{1}{l}{ 12}  & 14.87 & 16.22 & \textbf{7.18} & 7.84 \\
\multicolumn{1}{l}{ 8}   & 56.41 & 56.73 & \textbf{8.73} & 9.99 \\
\multicolumn{1}{l}{ 6}   & 188.08 & 168.99 & \textbf{10.71} & 12.98 \\
\multicolumn{1}{l}{ 4}   & 319.41 & 279.06 & \textbf{17.28} & 21.82 \\
\bottomrule
\end{tabular}
\label{tab:fid_nfe_all_32}
\end{minipage}
\caption{Comparing the FID vs. NFE on ImageNet32 for two baselines and two of our methods.}
\end{table}

\begin{table}[H]
\label{tab:fid_nfe_all_64}
\centering
\begin{minipage}{0.5\linewidth}
\begin{tabular}{ c r r r r }
\multicolumn{5}{c}{ ImageNet64 FID (Euler) }  \\ 
\toprule
\multicolumn{1}{l}{ NFE} & \multicolumn{1}{c}{ DDPM} & \multicolumn{1}{c}{ ScoreSDE} & \multicolumn{1}{c}{ BatchOT} & \multicolumn{1}{c}{ Stable} \\ \midrule
\multicolumn{1}{l}{ Adaptive \hspace{1em}}  & 13.80 & 26.64 & 12.37 & \textbf{11.82} \\
\multicolumn{1}{l}{ 40} & 25.83 & 44.16 & 14.79 & \textbf{13.39} \\
\multicolumn{1}{l}{ 20} & 66.42 & 82.97 & 17.06 & \textbf{15.15} \\
\multicolumn{1}{l}{ 12} & 158.46 & 141.79 & 20.94 & \textbf{18.81} \\
\multicolumn{1}{l}{ 8}  & 258.49 & 210.29 & 27.56 & \textbf{26.38} \\
\multicolumn{1}{l}{ 6}  & 321.04 & 262.20 & \textbf{36.17} & 37.14 \\
\multicolumn{1}{l}{ 4}  & 373.08 & 335.54 & \textbf{56.75} & 63.25 \\
\bottomrule
\end{tabular}
\end{minipage}%
\begin{minipage}{0.5\linewidth}
\begin{tabular}{ c r r r r }
\multicolumn{5}{c}{ ImageNet64 FID (Midpoint)} \\
\toprule
\multicolumn{1}{l}{ NFE} & \multicolumn{1}{c}{ DDPM} & \multicolumn{1}{c}{ ScoreSDE} & \multicolumn{1}{c}{ BatchOT} & \multicolumn{1}{c}{ Stable} \\ 
\midrule
\multicolumn{1}{l}{ Adaptive \hspace{1em}}  & 13.80 & 26.64 & 12.37 & \textbf{11.82} \\
\multicolumn{1}{l}{ 40}  & 15.3 & 26.67 & 14.22 & \textbf{12.97} \\
\multicolumn{1}{l}{ 20}  & 15.05 & 25.73 & 16.05 & \textbf{14.76} \\
\multicolumn{1}{l}{ 12}  & 18.91 & 29.99 & 18.27 & \textbf{17.60} \\
\multicolumn{1}{l}{ 8}   & 53.15 & 67.83 & \textbf{20.85} & 21.36 \\
\multicolumn{1}{l}{ 6}   & 179.79 & 155.91 & \textbf{24.87} & 27.15 \\
\multicolumn{1}{l}{ 4}   & 330.53 & 279.00 & \textbf{38.45} & 46.08 \\
\bottomrule
\end{tabular}
\end{minipage}
\caption{Comparing the FID vs. NFE on ImageNet64 for two baselines and two of our methods.}
\end{table}

\subsection{Runtime per iteration is not significantly affected by solving for couplings}\label{app:runtime_comparison}

\begin{table}[H]
\centering
\ra{1.1}
\resizebox{0.6\columnwidth}{!}{%
\begin{tabular}{@{} l c c c c @{}}\toprule
  & \multicolumn{2}{c}{\bf ImageNet 32$\times$32}
  & \multicolumn{2}{c}{\bf ImageNet 64$\times$64} \\
& {It./s} & {Rel. increase}  
& {It./s} & {Rel. increase} \\
\midrule
CondOT (reference) & 
1.16 & ---  & 
1.31 & ---  \\
BatchOT  & 
1.15 & 0.8\% & 
1.26 & 3.9\% \\
Stable & 
1.15 & 0.8\% & 
1.26 & 3.9\% \\
\bottomrule
\end{tabular}
}
\caption{ Absolute and relative runtime comparisons between CondOT, BatchOT and Stable matching. ``It./s" denotes the number of iterations per second, and ``Rel. increase" is the relative increase with respect to CondOT. Note that these are on relatively standard batch sizes (refer to \cref{app:exp} for exact batch sizes).}
\label{tab:relativeruntime_table}
\end{table}

\subsection{Convergence improves when using larger coupling sizes}

\begin{figure}[H]
    \centering
    \includegraphics[width=0.5\linewidth]{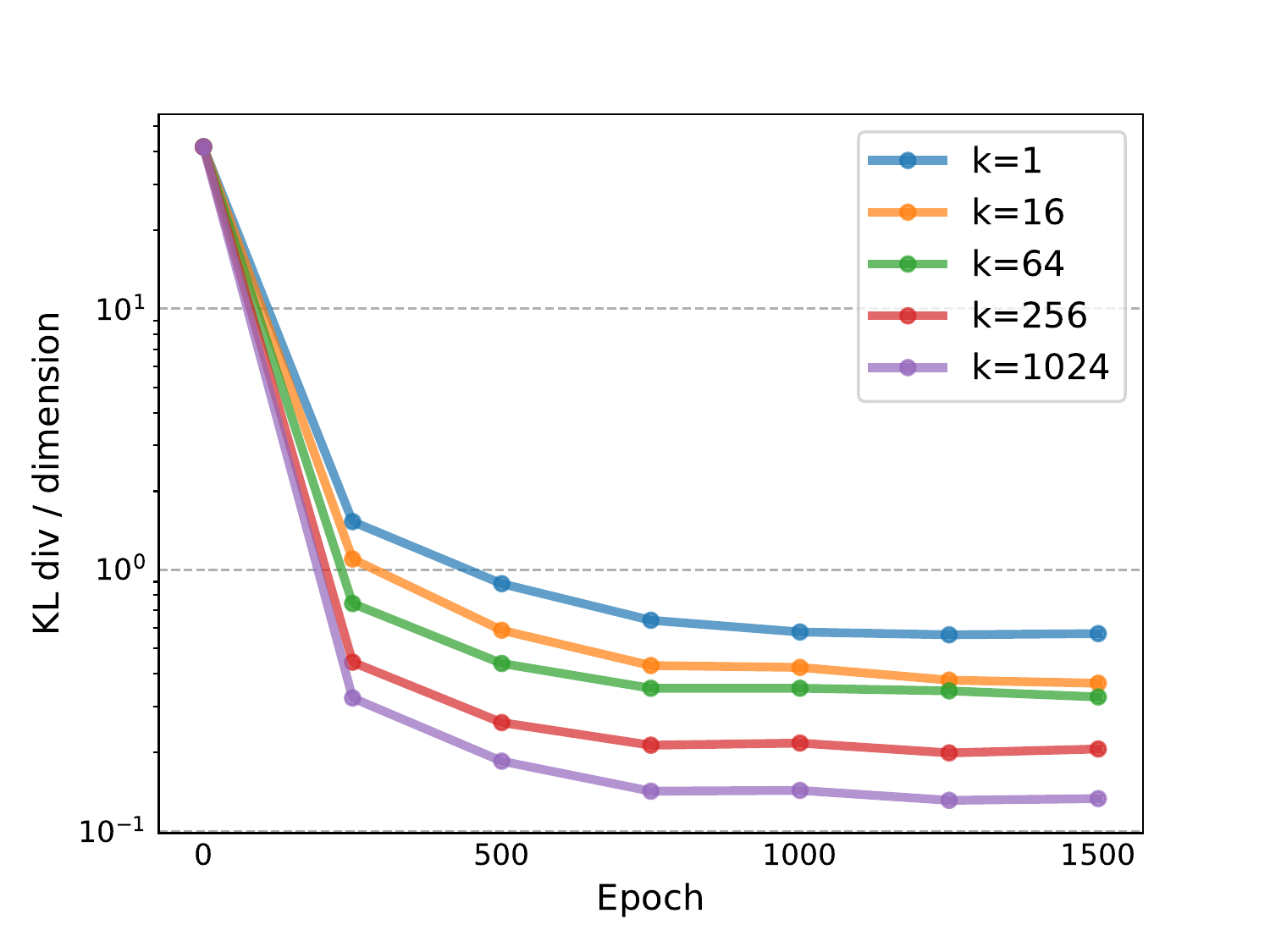}
    \vspace{-1.5em}
    \caption{Larger couplings sizes ($k$) for defining the multisample coupling results in faster and more stable convergence. This is done on the 64-D experiments in \cref{sec:improve_batchot}. The batch size (number of samples) for training is kept thestr same and only $k$ is varied for solving the couplings.}
    \label{fig:convergence_vs_k}
\end{figure}

\pagebreak
\section{Generated samples}

\begin{figure}[H]
    \centering
    \begin{subfigure}[b]{0.3\linewidth}
    {\tiny \; NFE=400 \;\;\;\;\;\;\;\;\;\;\;\;\; 12 \;\;\;\;\;\;\;\;\;\;\;\;\;\;\;\;\;\; 8 \;\;\;\;\;\;\;\;\;\;\;\;\;\;\;\;  6} \\
        \includegraphics[width=\linewidth]{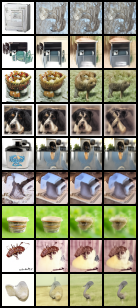}
        \vspace{-1.5em}
        \caption*{Flow Matching}
    \end{subfigure}
    \begin{subfigure}[b]{0.3\linewidth}
    {\tiny \; NFE=400 \;\;\;\;\;\;\;\;\;\;\;\;\; 12 \;\;\;\;\;\;\;\;\;\;\;\;\;\;\;\;\;\; 8 \;\;\;\;\;\;\;\;\;\;\;\;\;\;\;\;  6} \\
        \includegraphics[width=\linewidth]{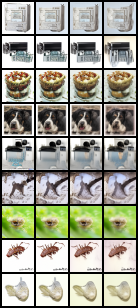}
        \vspace{-1.5em}
        \caption*{Multisample Flow Matching}
    \end{subfigure}
    \caption{Multisample Flow Matching trained with batch optimal couplings produces more consistent samples across varying NFEs on ImageNet32. From left to right, the NFEs used to generate these samples are 200, 12, 8, and 6 using a midpoint discretization. Note that both flows on each row start from the same noise sample.}
    \label{fig:samples_vs_nfe_imagenet32_paths}
\end{figure}

\begin{figure}[H]
    \centering
    \begin{subfigure}[b]{0.3\linewidth}
    {\tiny \; NFE=400 \;\;\;\;\;\;\;\;\;\;\;\; 12 \;\;\;\;\;\;\;\;\;\;\;\;\;\;\;\;\;\; 8 \;\;\;\;\;\;\;\;\;\;\;\;\;\;\;\;\;\; 6} \\
        \includegraphics[width=\linewidth]{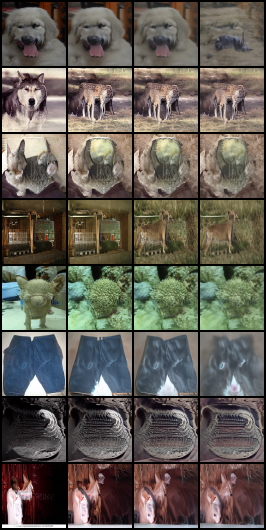}
        \vspace{-1.5em}
        \caption*{Flow Matching}
    \end{subfigure}
    \begin{subfigure}[b]{0.3\linewidth}
        {\tiny \; NFE=400 \;\;\;\;\;\;\;\;\;\;\;\; 12 \;\;\;\;\;\;\;\;\;\;\;\;\;\;\;\;\;\; 8 \;\;\;\;\;\;\;\;\;\;\;\;\;\;\;\;\;\; 6} \\
        \includegraphics[width=\linewidth]{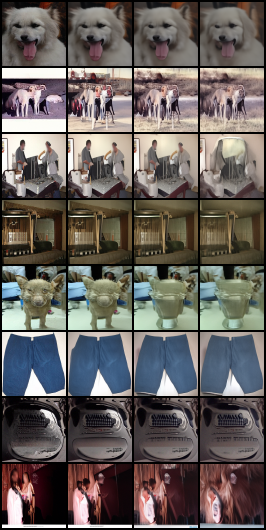}
        \vspace{-1.5em}
        \caption*{Multisample Flow Matching}
    \end{subfigure}
    \caption{Multisample Flow Matching trained with batch optimal couplings produces more consistent samples across varying NFEs on ImageNet64. From left to right, the NFEs used to generate these samples are 200, 12, 8, and 6 using a midpoint discretization. Note that both flows on each row start from the same noise sample.}
    \label{fig:samples_vs_nfe_imagenet64_morepaths}
\end{figure}

\begin{figure}[H]
    \centering
        \includegraphics[width=0.8\linewidth]{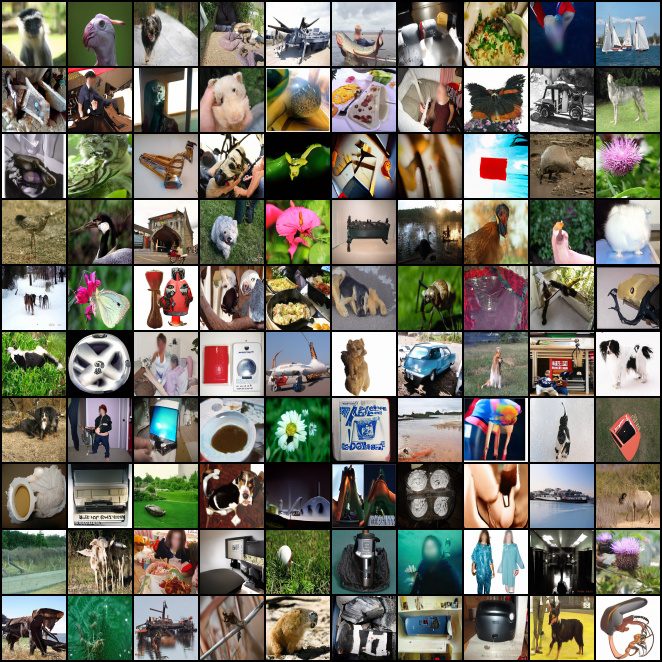}
    \caption{Non-curated generated images for ImageNet64 using Multisample Flow Matching with BatchOT coupling.}
    \label{fig: imagenet64_generated_batchot}
\end{figure}

\pagebreak
\section{Theorems and proofs}\label{sec:proofs}
\subsection{Proof of Lemma~\ref{lem:marginals}} \label{subsec:marginals}

We need only prove that the marginal probability path interpolates between $q_0$ and $q_1$. 
\begin{equation}
p_{0}(x) = \int p_{0}(x | x_1) q_1(x_1) dx_1 = \int q(x | x_1) q_1(x_1) dx_1 = q_0(x).
\end{equation}
Then since $u_t(x | x_1)$ transports all points $x \in \R^D$ to $x_1$ at time $t=1$, we satisfy $p_{t=1}(x | x_1) = \delta(x - x_1)$.
\begin{equation}
    p_{1}(x) = \int p_{1}(x | x_1) q_1(x_1) dx = \int \delta(x - x_1) q_1(x_1) dx_1 = q_1(x).
\end{equation}
Theorems 1 and 2 of \citet{lipman2022flow} can then be used to prove that (i) the marginal vector field $u_t(x)$ transports between $p_0 = q_0$ and $p_1 = q_1$, and (ii) the Joint CFM objective has the same gradient in expectation as the Flow Matching objective and is uniquely minimized by $v_t(x;\theta) = u_t(x)$.

\subsection{Proof of Lemma~\ref{lem:var_bound}} \label{subsec:grad_variance}
Note that

\begin{equation}\label{eq:vargrad}
\begin{split}
    &\text{Cov}_{p_t(x_1|x)} \left( \nabla_\theta \norm{v_t(x;\theta) - u_t(x | x_1)}^2 \right) =\text{Cov}_{p_t(x_1|x)} \left( \nabla_\theta \norm{v_t(x;\theta)}^2 - \left( \nabla_\theta v_t(x;\theta)\right)\tran u_t(x | x_0, x_1) \right) \\
    &=\left( \nabla_\theta v_t(x;\theta) \right)\tran \text{Cov}_{p_t(x_1|x)} \left( u_t(x | x_1) \right) \left( \nabla_\theta v_t(x;\theta) \right),
\end{split}
\end{equation}
and that 
\begin{align} \label{eq:cov_expression}
\text{Cov}_{p_t(x_1|x)} \left( u_t(x | x_1) \right) = \mathbb{E}_{p_t(x_1|x)} \left( u_t(x | x_1) - u_t(x) \right) \left( u_t(x | x_1) - u_t(x) \right)^{\top}.
\end{align}
Here, we used that $u_t(x) = \E_{p_t( x_1 | x)} \left[ u_t( x | x_1) \right]$ by \eqref{eq:marg_vf}.
If we take the trace on both sides of \eqref{eq:vargrad}, we get
\begin{align}
\begin{split}
    &\mathrm{Tr} \big[\text{Cov}_{p_t(x_1|x)} \left( \nabla_\theta \norm{v_t(x;\theta) - u_t(x | x_1)}^2 \right)\big] = \mathrm{Tr} \big[ \left( \nabla_\theta v_t(x;\theta)\right)\tran \text{Cov}_{p_t(x_1|x)} \left( u_t(x | x_1) \right) \left( \nabla_\theta v_t(x;\theta)\right) \big] \\ &= \mathrm{Tr} \big[ \text{Cov}_{p_t(x_1|x)} \left( u_t(x | x_1) \right) \left( \nabla_\theta v_t(x;\theta) \right) \left( \nabla_\theta v_t(x;\theta) \right)\tran \big] = \langle \text{Cov}_{p_t(x_1|x)} \left( u_t(x | x_1) \right), \left( \nabla_\theta v_t(x;\theta) \right) \left( \nabla_\theta v_t(x;\theta) \right)\tran \rangle_{F} \\ &\leq \| \text{Cov}_{p_t(x_1|x)} \left( u_t(x | x_1) \right)\|_{F} \| \left( \nabla_\theta v_t(x;\theta) \right) \left( \nabla_\theta v_t(x;\theta) \right)\tran \|_{F} 
    \\ &\leq \mathbb{E}_{p_t(x_1|x)} \| \left( u_t(x | x_1) - u_t(x) \right) \left( u_t(x | x_1) - u_t(x) \right)^{\top} \|_{F} \| \left( \nabla_\theta v_t(x;\theta) \right) \left( \nabla_\theta v_t(x;\theta) \right)^{\top} \|_{F} \\ &= 
    \| \nabla_\theta v_t(x;\theta) \|^2 \mathbb{E}_{p_t(x_1|x)} \| u_t(x | x_1) - u_t(x) \|^2.
\end{split}
\end{align}
The second equality holds because $\mathrm{Tr}(AB) = \mathrm{Tr}(BA)$ when both expressions are well defined, and the third equality holds by the definition of the Frobenius inner product $\langle \cdot, \cdot \rangle_{F}$. The first inequality holds by the Cauchy-Schwarz inequality. The second inequality holds by equation \eqref{eq:cov_expression} and by the triangle inequality.
In the last equality we used that for any vector $v$, $\|  v v^{\top}\|_{F} = (\mathrm{Tr}(v v^{\top},v v^{\top}))^{1/2} = \|v\|^2$. This proves \eqref{eq:bound_tr_cov}. 

To prove \eqref{eq:avg_total_variance}, we write:
\begin{align}
\begin{split}
    &\E_{t, p_t(x)}[\sigma^2_{t,x}] \\ &\leq 
    \E_{t, p_t(x)}[\| \nabla_\theta v_t(x;\theta) \|^2 \mathbb{E}_{p_t(x_1|x)} \| u_t(x | x_1) - u_t(x) \|^2]
    \\ &\leq 
    \max_{x,t} \| \nabla_\theta v_t(x;\theta) \|^2 \times \E_{t, p_t(x)}[ \mathbb{E}_{p_t(x_1|x)} \| u_t(x | x_1) - u_t(x) \|^2]
    \\ &= 
    \max_{x,t} \| \nabla_\theta v_t(x;\theta) \|^2 \times \E_{t, q(x_0,x_1)}[ \| u_t(x_t | x_1) - v_t(x_t;\theta) \|^2] \leq \max_{t,x} \norm{\nabla_\theta v_t(x;\theta)}^2 \times \mathcal{L}_\mathrm{JCFM}
\end{split}
\end{align}
Here, the first inequality holds by \eqref{eq:bound_tr_cov}, and the last inequality holds because $u_t(x)$ is the minimizer of $\mathcal{L}_\mathrm{JCFM}$.  

\subsection{Proof of Lemma~\ref{lem:q_marginals}} \label{subsec:q_marginals}
For an arbitrary test function $f$, by the construction of $q$ we write
\begin{align}
    \mathbb{E}_{q(x_0,x_1)} f(x_0) = \mathbb{E}_{\smash{\{x_0^{(i)}\}_{i=1}^k \sim q_0},\smash{\{x_1^{(i)}\}_{i=1}^k \sim q_1}} \mathbb{E}_{q^k(x_0,x_1)} f(x_0).
\end{align}
Since $q^k$ has marginal $\frac{1}{k} \sum_{i=1}^{k} \delta(x_0 - x_0^{(i)})$ because $\pi$ is a doubly stochastic matrix, we obtain that $\mathbb{E}_{q^k(x_0,x_1)} f(x_0) = \frac{1}{k} \sum_{i=1}^{k} f(x_0^{(i)})$ and then the right-hand side is equal to
\begin{align}
    \mathbb{E}_{\smash{\{x_0^{(i)}\}_{i=1}^k \sim q_0},\smash{\{x_1^{(i)}\}_{i=1}^k \sim q_1}} \frac{1}{k} \sum_{i=1}^{k} f(x_0^{(i)}) = \mathbb{E}_{q_0(x_0)} f(x_0),
\end{align}
which proves that the marginal of $q$ for $x_0$ is $q_0$.
The same argument works for the $x_1$ marginal.

\subsection{Proof of Theorem~\ref{thm:limiting}} \label{sec:proof_limiting}

\paragraph{Notation} We begin by recalling and introducing some additional notation. Let $\bm{X}_0 = (x_0^i)_{i=1}^{+\infty}$, $\bm{X}_1 = (x_1^i)_{i=1}^{+\infty}$ be sequences of i.i.d. samples from the distributions $q_0$ and $q_1$, and denote by $\bm{X}_0^k = (x_0^i)_{i=1}^{k}$, $\bm{X}_1^k = (x_1^i)_{i=1}^{k}$ the finite sequences containing the initial $k$ samples. We denote by $q_0^k$ and $q_1^k$ the empirical distributions corresponding to $\bm{X}_0^k$ and $\bm{X}_1^k$, i.e. $q_0^k = \frac{1}{k} \sum_{i=1}^{k} \delta_{x_0^i}$, $q_1^k = \frac{1}{k} \sum_{i=1}^{k} \delta_{x_1^i}$. Let $q^k$ be the distribution over $\R^d \times \R^d$ which is output by the matching algorithm; $q^k$ has marginals that are equal to $q_0^k$ and $q_1^k$. Let $q^*$ be the optimal transport plan between $q_0$ and $q_1$,  and let $\Tilde{q}^k$ be the optimal transport plan between $q^k$ and $q$ under the quadratic cost. Using this additional notation, we rewrite some of the objects that were defined in the main text in a lengthier, more precise way:
\begin{enumerate}[label=(\roman*)]
    \item The marginal vector field corresponding to sample size $k$: 
    \begin{align} \label{eq:u_t_k}
        u_t^{k}(x) = \mathbb{E}_{\bm{X}_0^k \overset{\mathrm{iid}}{\sim} q_0,\bm{X}_1^k \overset{\mathrm{iid}}{\sim} q_1,(x_0,x_1) \sim q^k} [x_1 - x_0 | x = tx_1 + (1-t) x_0 ], \qquad \forall t \in [0,1]. 
    \end{align}
    We made the dependency on $k$ explicit, and we used that $\psi_t(x_0|x_1) = tx_1 + (1-t) x_0$. Note that equivalently, we can write $u_t^{k}$ as the solution of a simple variational problem.
    \begin{align} \label{eq:u_t_k_min}
        u_t^k = \argmin_{u_t} \mathbb{E}_{\bm{X}_0^k \overset{\mathrm{iid}}{\sim} q_0,\bm{X}_1^k \overset{\mathrm{iid}}{\sim} q_1,(x_0,x_1) \sim q^k} \|x_1 - x_0 - u_t(tx_1 + (1-t) x_0)\|^2, \qquad \forall t \in [0,1].
    \end{align}
    \item The flow $\phi_t^k(x_0)$ corresponding to $u_t^k$, i.e. the solution of $\frac{dx_t}{dt} = u_t^k(x_t)$ with initial condition $x_0$. We made the dependency on $k$ explicit.
    \item The straightness of the flow $\phi_t^k$:
    \begin{align} \label{eq:straightness_k}
        S^k = \mathbb{E}_{t \sim \mathrm{U}(0,1), x_0 \sim q_0} \big[ \|u_t^{k}( \phi_t^k(x_0))\|^2 - \| \phi_1^{k}(x_0) - x_0\|^2 \big].
    \end{align}
\end{enumerate}

\paragraph{Assumptions} We will use the following three assumptions, which allow us to potentially extend our result beyond BatchOT:
\begin{description}[noitemsep,topsep=0pt] 
\item \textbf{(A1)} The distributions $q_0$ and $q_1$ over $\R^d$ have bounded supports, i.e. there exists $C > 0$ such that for any $x \in \mathrm{supp}(q_0) \cup \mathrm{supp}(q_1)$, $\|x\| \leq C$.
\item \textbf{(A2)} $q_0$ admits a density and the optimal transport map $T$ between $q_0$ and $q_1$ under the quadratic cost is continuous.
\item \textbf{(A3)} We assume that almost surely w.r.t. the draw of $\bm{X}_0$ and $\bm{X}_1$, $q^k$ converges weakly to $q$ as $k \to \infty$.
\end{description}
Some comments are in order as to when assumptions \textbf{(A2)}, \textbf{(A3)} hold, since they are not directly verifiable. By the Caffarelli regularity theorem (see \citet{villani2008optimal}, Ch.~12, originally in \citet{caffarelli1992theregularity}), a sufficient condition for \textbf{(A2)} to hold is the following:
\begin{description}[noitemsep,topsep=0pt]
\item \textbf{(A2')} $q_0$ and $q_1$ have a common support $\Omega$ which is compact and convex, have $\alpha$-Hölder densities, and they satisfy the lower bound $q_0,q_1 > \gamma$ for some $\gamma > 0$.
\end{description}

Assumption \textbf{(A3)} holds when the matching algorithm is BatchOT, that is, when $q^k$ is the optimal transport plan between $q_0^k$ and $q_1^k$, as shown by the following proposition, which is proven in App.~\ref{subsec:batchOT_works}.
\begin{proposition} \label{prop:almost_sure_weak_convergence}
    Let $q^k$ be the optimal transport plan between $q_0^k$ and $q_1^k$ under the quadratic cost (i.e. the result of Steps [\ref{item:1}-\ref{item:3}] under BatchOT).
    We have that almost surely w.r.t. the draws of $\bm{X}_0$ and $\bm{X}_1$, the sequence $(q_k)_{k \geq 0}$ converges weakly to $q^*$, i.e. assumption \textbf{(A3)} holds.
\end{proposition}

\paragraph{Proof structure} We split the proof of Theorem~\ref{thm:limiting} into two parts: in Subsubsec.~\ref{subsec:optim_cfm} we prove that the optimal value of the Joint CFM objective \eqref{eq:cfm_joint} converges to zero as $k \to \infty$. In Subsubsec.~\ref{subsec:conv_straightness}, we prove that the straightness converges to zero and the transport cost converges to the optimal transport cost as $k \to \infty$.

\subsubsection{Convergence of the optimal value of the CFM objective} \label{subsec:optim_cfm}
\begin{theorem} \label{thm:convergence_variance}
Suppose that assumptions \textbf{(A1)}, \textbf{(A2)} and \textbf{(A3)} hold.
We have that
\begin{align}
    \lim_{k \to \infty}
    \mathbb{E}_{t \sim \mathrm{U}(0,1), \bm{X}_0^k \overset{\mathrm{iid}}{\sim} q_0,\bm{X}_1^k \overset{\mathrm{iid}}{\sim} q_1, (x_0,x_1) \sim q^k} \|x_1 - x_0 - u_t^{k}(tx_1 + (1-t) x_0)\|^2 = 0,
\end{align}
where $u_t^k$ is the marginal vector field as defined in \eqref{eq:u_t_k}.
\end{theorem}
\begin{proof}
    The transport plan $q^*$ satisfies the non-crossing paths property, 
    that is, for each $x \in \R^d$ and $t \in [0,1]$, there exists at most one pair $(x_0$, $x_1)$ such that $x = tx_1 + (1-t) x_0$ \citep{nurbekyan2020nocollision,villani2003topics}. Consequently, when such a pair $(x_0'$, $x_1')$ exists, we have that the analogue of the vector field in \eqref{eq:u_t_k} admits a simple expression: 
    \begin{align} \label{eq:u_t_star}
        u_t^{*}(x) := \mathbb{E}_{(x_0,x_1) \sim q^*} [x_1 - x_0 | x = tx_1 + (1-t) x_0 ] = x_1' - x_0'
    \end{align}
    This directly implies that
    \begin{align}
        \mathbb{E}_{(x_0,x_1) \sim q^*} \|x_1 - x_0 - u_t^*(tx_1 + (1-t) x_0)\|^2 = 0.
    \end{align}
    Applying this, we can write
    \begin{align} \begin{split} \label{eq:expectation_u_t_star}
        &\mathbb{E}_{t \sim \mathrm{U}(0,1), (x_0,x_1) \sim q^k} \|x_1 - x_0 - u_t^{*}(tx_1 + (1-t) x_0)\|^2 \\ &= |\mathbb{E}_{(x_0,x_1) \sim q^k} [\mathbb{E}_{t \sim \mathrm{U}(0,1)}  \|x_1 - x_0 - u_t^{*}(tx_1 + (1-t) x_0)\|^2 ]\\ &\quad\ - \mathbb{E}_{(x_0,x_1) \sim q^*}[\mathbb{E}_{t \sim \mathrm{U}(0,1)} \|x_1 - x_0 - u_t^{*}(tx_1 + (1-t) x_0)\|^2]|.
    \end{split}
    \end{align}
    Now, define the function $f : \mathrm{supp}(q_0) \times \mathrm{supp}(q_1) \to \R$ as
    \begin{align} \label{eq:f_def}
        f(x_0,x_1) = \mathbb{E}_{t \sim \mathrm{U}(0,1)} \|x_1 - x_0 - u_t^{*}(tx_1 + (1-t) x_0)\|^2.
    \end{align}
    By Lemma \ref{lem:f_continuous_bounded}, which holds under \textbf{(A1)} and \textbf{(A2)}, we have that $f$ is bounded and continuous. 
    Assumption \textbf{(A3)} states that almost surely w.r.t. the draws of $\bm{X}_0$ and $\bm{X}_1$, the measure $q_k$ converges weakly to $q^*$.
    We apply the definition of weak convergence of measures, which implies that almost surely,
    \begin{align}
        \lim_{k \to \infty} \mathbb{E}_{(x_0,x_1) \sim q^k}[f(x_0,x_1)] = \mathbb{E}_{(x_0,x_1) \sim q}[f(x_0,x_1)].
    \end{align}
    Equivalently, the right-hand side of \eqref{eq:expectation_u_t_star} converges to zero as $k$ tends to infinity. Hence, $\mathbb{E}_{t \sim \mathrm{U}(0,1), (x_0,x_1) \sim q^k} \|x_1 - x_0 - u_t^{*}(tx_1 + (1-t) x_0)\|^2 \to 0$ almost surely. Almost sure convergence implies convergence in probability, which means that
    \begin{align}
        \mathrm{Pr}(\mathbb{E}_{t \sim \mathrm{U}(0,1), (x_0,x_1) \sim q^k} \|x_1 - x_0 - u_t^{*}(tx_1 + (1-t) x_0)\|^2 > \epsilon) \xrightarrow[]{k \to \infty} 0, \quad \forall \epsilon > 0.
    \end{align}
    Here, the randomness comes only from drawing the random variables $ \bm{X}_0^k,\bm{X}_1^k$.
    Also, using again that $f$ is bounded, say by the constant $C > 0$, we can write $\mathbb{E}_{t \sim \mathrm{U}(0,1), (x_0,x_1) \sim q^k} \|x_1 - x_0 - u_t^{*}(tx_1 + (1-t) x_0)\|^2 \leq C$, for all $k \geq 0$. A crude bound yields
    \begin{align}
        &\mathbb{E}_{t \sim \mathrm{U}(0,1), \bm{X}_0^k,\bm{X}_1^k, (x_0,x_1) \sim q^k} \|x_1 - x_0 - u_t^{*}(tx_1 + (1-t) x_0)\|^2 \\ &\leq \epsilon + C \mathrm{Pr}(\mathbb{E}_{t \sim \mathrm{U}(0,1), (x_0,x_1) \sim q^k} \|x_1 - x_0 - u_t^{*}(tx_1 + (1-t) x_0)\|^2 > \epsilon).
    \end{align}
    In this equation and from now, we write $\bm{X}_0^k,\bm{X}_1^k$ instead of $\bm{X}_0^k \overset{\mathrm{iid}}{\sim} q_0,\bm{X}_1^k \overset{\mathrm{iid}}{\sim} q_1$ for shortness.
    We can take $\epsilon$ arbitrarily small, and for a given $\epsilon$ we can make the second term in the right-hand side arbitrarily small by taking $k$ large enough. Hence, we obtain that
    \begin{align}
        \lim_{k \to \infty} \mathbb{E}_{t \sim \mathrm{U}(0,1), \bm{X}_0^k,\bm{X}_1^k, (x_0,x_1) \sim q^k} \|x_1 - x_0 - u_t^{*}(tx_1 + (1-t) x_0)\|^2 = 0.
    \end{align}
    To conclude the proof, we use the variational characterization of $u_t^k$ given in \eqref{eq:u_t_k_min}, which implies that
    \begin{align}
    \begin{split}
        &\mathbb{E}_{t \sim \mathrm{U}(0,1), \bm{X}_0^k,\bm{X}_1^k, (x_0,x_1) \sim q^k} \|x_1 - x_0 - u_t^{k}(tx_1 + (1-t) x_0)\|^2 \\ &\leq \mathbb{E}_{t \sim \mathrm{U}(0,1), \bm{X}_0^k,\bm{X}_1^k, (x_0,x_1) \sim q^k} \|x_1 - x_0 - u_t^{*}(tx_1 + (1-t) x_0)\|^2 \to 0.
    \end{split}
    \end{align}
\end{proof}

\begin{lemma} \label{lem:f_continuous_bounded}
    Let $f$ be the function defined in equation \eqref{eq:f_def}. 
    Suppose that assumptions \textbf{(A1)} and \textbf{(A2)} hold.
    Then, $f$ is bounded and continuous.
\end{lemma}
\begin{proof}
    First, we show that the function $u^*_t$ defined in equation \eqref{eq:u_t_star} is bounded and continuous wherever it is defined. It is bounded because $u_t^{*}(x) = x_1' - x_0'$ for some $x_0'$ in $\mathrm{supp}(q_0)$ and $x_1'$ in $\mathrm{supp}(q_1)$, which are both bounded by assumption. 
    
    To show that $u^*_t$ is continuous, we use that $q_0$ is absolutely continuous and that consequently a transport map $T$ exists. Moreover, we have that $x_1' = T(x_0')$. 
    Consider the transport map $T_t$ at time $t$, defined as $T_t(x) = t T(x) + (1-t) x$. 
    Thus, we can write that 
    $u_t^{*}(T_t(x_0)) = T(x_0) - x_0$. The non-crossing paths property implies that $T_t$ is invertible, which means that an inverse $T_t^{-1}$ exists. We can write 
    \begin{align} \label{eq:u_t_star_rewritten}
        u_t^{*}(x) = T(T_t^{-1}(x)) - T_t^{-1}(x).
    \end{align}
    By assumption \textbf{(A2)}, the transport map $T$ is continuous, and so is $T_t$. It is well-known fact that if $E,E'$ are metric spaces, $E$ is compact, and $f:E \rightarrow E'$ a continuous bijective function, then $f^{-1} : E' \rightarrow E$ is continuous. Thus, $T_t^{-1}$ is also continuous. From equation \eqref{eq:u_t_star_rewritten}, we conclude that $u_t^{*}$ is continuous.

    The rest of the proof is straightforward: $(x_1, x_0) \mapsto \|x_1 - x_0 - u_t^{*}(tx_1 + (1-t) x_0)\|^2$ is bounded and continuous on the bounded supports of $q_0$ and $q_1$ for all $t \in [0,1]$, and then $f$ is also continuous and bounded since it is an average of continuous bounded functions, applying the dominated convergence theorem.
\end{proof}
\subsubsection{Convergence of the straightness and the transport cost} \label{subsec:conv_straightness}

\begin{theorem}
Suppose that assumptions \textbf{(A1)} and \textbf{(A3)} hold. 
Then,
\begin{enumerate}[label=(\roman*)]
\item We have that $\lim_{k \to \infty} S^k = 0$, where $S^k$ is the straightness defined in \eqref{eq:straightness_k}.
\item We also have that
\begin{align} 
\label{eq:ii_inequalities}
    &\mathbb{E}_{t \sim \mathrm{U}(0,1), x_0 \sim q_0} \|u_t^{k}( \phi_t^k(x_0))\|^2 \geq \mathbb{E}_{x_0 \sim q_0} \| \phi_1^{k}(x_0) - x_0\|^2 \geq W^2_2(q_0,q_1), \\
    &\lim_{k \to \infty} \mathbb{E}_{t \sim \mathrm{U}(0,1), x_0 \sim q_0} \|u_t^{k}( \phi_t^k(x_0))\|^2 = \lim_{k \to \infty} \mathbb{E}_{x_0 \sim q_0} \| \phi_1^{k}(x_0) - x_0\|^2 = W^2_2(q_0,q_1).
    \label{eq:ii_equalities}
\end{align}
\end{enumerate}
\end{theorem}
\begin{proof}
We begin with the proof of (i). We introduce some additional notation.
We define the quantity $S^*$ in analogy with $S^k$:
\begin{align}
S^* = \mathbb{E}_{t \sim \mathrm{U}(0,1), x_0 \sim q_0} \big[ \|u_t^{*}( \phi_t^*(x_0))\|^2 - \| \phi_1^{*}(x_0) - x_0\|^2 \big],
\end{align}
and $\phi_t^*(x_0)$ as the solution of the ODE $\frac{dx_t}{dt} = u_t^*(x_t)$. Since the trajectories for the optimal transport vector field are straight lines, we deduce from the alternative expression of the straightness (equation \eqref{eq:straightness_2}) that $S^* = 0$. An alternative way to see this is by the Benamou-Brenier theorem \citep{benamou2000computational}, which states that the dynamic optimal transport cost $\mathbb{E}_{t \sim \mathrm{U}(0,1), x_0 \sim q_0} \|u_t^{*}( \phi_t^*(x_0))\|^2$ is equal to the static optimal transport cost $\mathbb{E}_{t \sim \mathrm{U}(0,1), x_0 \sim q_0} \| \phi_1^{*}(x_0) - x_0\|^2$.

We will first prove that $\mathbb{E}_{t \sim \mathrm{U}(0,1), x_0 \sim q_0} \|u_t^{k}( \phi_t^k(x_0))\|^2$ converges to $\mathbb{E}_{t \sim \mathrm{U}(0,1), x_0 \sim q_0} \|u_t^{*}( \phi_t^*(x_0))\|^2$ and then that $\mathbb{E}_{t \sim \mathrm{U}(0,1), x_0 \sim q_0} \| \phi_1^{k}(x_0) - x_0\|^2$ converges to $\mathbb{E}_{t \sim \mathrm{U}(0,1), x_0 \sim q_0} \| \phi_1^{*}(x_0) - x_0\|^2$.

For given instances of $\bm{X}_0^k$ and $\bm{X}_1^k$, let $\tilde{q}_k$ be the optimal transport plan between the optimal transport plans $q$ and $q^k$. In other words, $\tilde{q}_k$ is a measure over the variables $x_0, x_1, x_0', x_1'$, and is such that its marginal w.r.t. $x_0, x_1$ is $q$, while its marginal w.r.t. $x_0', x_1'$ is $q^k$.
That is, we will use that for all $t \in [0,1]$, the random variable $t x_1 + (1-t) x_0$, with $(x_0, x_1) \sim q^k$, and $q^k$ built randomly from $\bm{X}_0^k \overset{\mathrm{iid}}{\sim} q_0,\bm{X}_1^k \overset{\mathrm{iid}}{\sim} q_1$, has the same distribution as the random variable $\phi^k_t(x_0)$, with $x_0 \sim q_0$. This is a direct consequence of Lemma~\ref{lem:marginals}, i.e. the marginal vector field $u_t$ generates the marginal probability path $p_t$.
An analogous statement holds for $q$, i.e. the random variable $t x_1 + (1-t) x_0$, with $(x_0, x_1) \sim q$, has the same distribution as the random variable $\phi^{*}_t(x_0)$, with $x_0 \sim q_0$. However, in this case it can be obtained immediately by the non-crossing paths property of the optimal transport plan. Hence,
\begin{align}
\begin{split} \label{eq:same_marginals}
    \mathbb{E}_{t \sim \mathrm{U}(0,1), x_0 \sim q_0} \|u_t^{*}( \phi_t^*(x_0))\|^2 &= \mathbb{E}_{t \sim \mathrm{U}(0,1), (x_0, x_1) \sim q} \| u_t^{*}(t x_1 + (1-t) x_0) \|^2, \\
    \mathbb{E}_{t \sim \mathrm{U}(0,1), x_0 \sim q_0} \|u_t^{k}( \phi_t^k(x_0))\|^2 &= \mathbb{E}_{t \sim \mathrm{U}(0,1), \bm{X}_0^k, \bm{X}_1^k, (x_0, x_1) \sim q^k} \| u_t^{k} (t x_1 + (1-t) x_0) \|^2.
\end{split}
\end{align}
Using this and the definition of $\tilde{q}^k$, and applying Jensen's inequality,
the Cauchy-Schwarz inequality and the triangle inequality, we can write
\begin{align}
\begin{split} \label{eq:u_star_u_k}
    &\big|\mathbb{E}_{t \sim \mathrm{U}(0,1), x_0 \sim q_0} \|u_t^{*}( \phi_t^*(x_0))\|^2 - \mathbb{E}_{t \sim \mathrm{U}(0,1), x_0 \sim q_0} \|u_t^{k}( \phi_t^k(x_0))\|^2 \big| \\ &= |\mathbb{E}_{t \sim \mathrm{U}(0,1), (x_0, x_1) \sim q} \| u_t^{*}(t x_1 + (1-t) x_0) \|^2 - \mathbb{E}_{t \sim \mathrm{U}(0,1), \bm{X}_0^k, \bm{X}_1^k, (x_0', x_1') \sim q^k} \| u_t^{k}(t x_1' + (1-t) x_0') \|^2| \\ &= \big| \mathbb{E}_{t \sim \mathrm{U}(0,1), \bm{X}_0^k, \bm{X}_1^k, (x_0, x_1,x_0',x_1') \sim \tilde{q}^k} \big[ \| u_t^{*} (t x_1 + (1-t) x_0) \|^2 - \| u_t^{k} (t x_1' + (1-t) x_0') \|^2 \big] \big|
    \\ &= \big| \mathbb{E}_{t \sim \mathrm{U}(0,1), \bm{X}_0^k, \bm{X}_1^k, (x_0, x_1,x_0',x_1') \sim \tilde{q}^k} \big[ (\|u_t^{*} (t x_1 + (1-t) x_0) \| - \| u_t^{k}(t x_1' + (1-t) x_0') \|) \\ &\qquad\qquad\qquad\qquad\qquad\qquad\qquad\qquad \times (\| u_t^{*}(t x_1 + (1-t) x_0) \| + \| u_t^{k}(t x_1' + (1-t) x_0') \|) \big] \big| \\ &\leq \big( \mathbb{E}_{t \sim \mathrm{U}(0,1), \bm{X}_0^k, \bm{X}_1^k, (x_0, x_1,x_0',x_1') \sim \tilde{q}^k} \big[ (\| u_t^{*}(t x_1 + (1-t) x_0) \| - \| u_t^{k}(t x_1' + (1-t) x_0') \| )^2 \big] \big)^{1/2} 
    \\ &\qquad\qquad\qquad\qquad \times \big( \mathbb{E}_{t \sim \mathrm{U}(0,1), \bm{X}_0^k, \bm{X}_1^k, (x_0, x_1,x_0',x_1') \sim \tilde{q}^k} \big[ (\| u_t^{*}(t x_1 + (1-t) x_0) \| + \| u_t^{k}(t x_1' + (1-t) x_0') \| )^2 \big] \big)^{1/2} \\ &\leq \big( \mathbb{E}_{t \sim \mathrm{U}(0,1), \bm{X}_0^k, \bm{X}_1^k, (x_0, x_1,x_0',x_1') \sim \tilde{q}^k} \| u_t^{*}(t x_1 + (1-t) x_0) - u_t^{k}(t x_1' + (1-t) x_0') \|^2 \big)^{1/2} 
    \\ &\qquad\qquad\qquad\qquad \times \big( \mathbb{E}_{t \sim \mathrm{U}(0,1), \bm{X}_0^k, \bm{X}_1^k, (x_0, x_1,x_0',x_1') \sim \tilde{q}^k} \big[ (\| u_t^{*}(t x_1 + (1-t) x_0 )\| + \| u_t^{k}(t x_1' + (1-t) x_0' )\| )^2 \big] \big)^{1/2}.
\end{split}
\end{align}
Remark that the second factor in the right-hand side is bounded because $u_t^*$ and $u_t^k$ are bounded. Using Lemma~\ref{lem:tilde_q_k}, we obtain that the first factor in the right-hand side tends to zero as $k$ grows. Thus,
\begin{align} \label{eq:first_term_S}
    \big|\mathbb{E}_{t \sim \mathrm{U}(0,1), x_0 \sim q_0} \|u_t^{*}( \phi_t^*(x_0))\|^2 - \mathbb{E}_{t \sim \mathrm{U}(0,1), x_0 \sim q_0^k} \|u_t^{k}( \phi_t^k(x_0))\|^2 \big| \xrightarrow[]{k \to \infty} 0.
\end{align}
Now, since $\mathbb{E}_{x_0 \sim q_0} \|\phi_1^{*}(x_0) - x_0\|^2$ is the optimal cost and $S^* = 0$, we write
\begin{align} \label{eq:second_term_S_0}
    \big| \mathbb{E}_{x_0 \sim q_0} \|\phi_1^{*}(x_0) - x_0\|^2 - \mathbb{E}_{x_0 \sim q_0} \|\phi_1^{k}(x_0) - x_0\|^2 \big| &= \mathbb{E}_{x_0 \sim q_0} \|\phi_1^{k}(x_0) - x_0\|^2 - \mathbb{E}_{x_0 \sim q_0} \|\phi_1^{*}(x_0) - x_0\|^2 \\ &= \mathbb{E}_{x_0 \sim q_0} \|\phi_1^{k}(x_0) - x_0\|^2 - \mathbb{E}_{t \sim \mathrm{U}(0,1), x_0 \sim q_0} \|u_t^{*}( \phi_t^*(x_0))\|^2. 
\end{align}
Since $\phi_t^{k}$ is the flow of $u_t^k$ and by Jensen's inequality, we have that
\begin{align*} 
    \mathbb{E}_{x_0 \sim q_0} \|\phi_1^{k}(x_0) - x_0\|^2 &= \mathbb{E}_{x_0 \sim q_0} \bigg\|\int_0^1 u_s^{k}( \phi_s^k(x_0')) \, ds \bigg\|^2 \\
    &\leq \mathbb{E}_{x_0 \sim q_0} \int_0^1 \| u_s^{k}( \phi_s^k(x_0))\|^2  \, ds = \mathbb{E}_{t \sim U(0,1),x_0 \sim q_0} \| u_t^{k}( \phi_t^k(x_0))\|^2. 
\end{align*}
Plugging this into \eqref{eq:second_term_S_0}, we get that
\begin{align} \label{eq:second_term_S}
    &\big| \mathbb{E}_{x_0 \sim q_0} \|\phi_1^{*}(x_0) - x_0\|^2 - \mathbb{E}_{x_0 \sim q_0} \|\phi_1^{k}(x_0) - x_0\|^2 \big| \\ &\leq \mathbb{E}_{t \sim U(0,1),x_0 \sim q_0} \| u_t^{k}( \phi_t^k(x_0))\|^2 - \mathbb{E}_{t \sim \mathrm{U}(0,1), x_0 \sim q_0} \|u_t^{*}( \phi_t^*(x_0))\|^2 \xrightarrow[]{k \to \infty} 0,
\end{align}
where the limit holds by \label{eq:second_term_S}.
Putting together \eqref{eq:first_term_S} and \eqref{eq:second_term_S}, we end up with $S^k = |S^{*}-S^k| \xrightarrow[]{k \to \infty} 0$, which proves (i).

We prove (ii).
The first inequality in \eqref{eq:ii_inequalities} holds because $S^k \geq 0$ since it can be written in a form analogous to \eqref{eq:straightness_2}. The second inequality in \eqref{eq:ii_inequalities} holds because $\mathbb{E}_{x_0 \sim q_0} \| \phi_1^{k}(x_0) - x_0\|^2$ is the squared transport cost for the map $x \mapsto \phi_1^{k}(x)$, which must be at least as large as the optimal cost. 
The first equality in \eqref{eq:ii_inequalities} is a direct consequence of (i). To prove the second equality in \eqref{eq:ii_inequalities}, we remark that $W_2^2(q_0,q_1) = \mathbb{E}_{x_0 \sim q_0} \| \phi_1^{*}(x_0) - x_0\|^2$. Then, equation \eqref{eq:second_term_S} readily implies that $|\mathbb{E}_{x_0 \sim q_0^k} \|\phi_1^{k}(x_0) - x_0\|^2 - W_2^2(q_0,q_1)| \xrightarrow[]{k \to \infty} 0$.
\end{proof}

\begin{lemma} \label{lem:tilde_q_k}
    Suppose that assumptions \textbf{(A1)} and \textbf{(A3)} hold. Let $\tilde{q}^k$ be the optimal transport plan between the optimal transport plans $q$ and $q^k$. We have that
    \begin{align}
        \lim_{k \to \infty} \mathbb{E}_{t \sim \mathrm{U}(0,1), \bm{X}_0^k \overset{\mathrm{iid}}{\sim} q_0,\bm{X}_1^k \overset{\mathrm{iid}}{\sim} q_1, (x_0, x_1, x_0', x_1') \sim \tilde{q}^k} \big[ \|u_t^{*}( t x_1 + (1-t) x_0 ) - u_t^{k}( t x_1' + (1-t) x_0' )\|^2 \big] = 0
    \end{align}
\end{lemma}

\begin{proof}
For given instances of $\bm{X}_0^k$ and $\bm{X}_1^k$, we can write
\begin{align} 
\begin{split} \label{eq:u_k_u_star_l2}
     &\mathbb{E}_{t \sim \mathrm{U}(0,1), (x_0, x_1, x_0', x_1') \sim \tilde{q}^k} \big[ \|u_t^{*}( t x_1 + (1-t) x_0 ) - u_t^{k}( t x_1' + (1-t) x_0' )\|^2 \big] \\ &=  \mathbb{E}_{t \sim \mathrm{U}(0,1), (x_0, x_1, x_0', x_1') \sim \tilde{q}^k} \big[ \|\mathbb{E}_{\tilde{x}_0,\tilde{x}_1 \sim q} [\tilde{x}_1 - \tilde{x}_0 | tx_1 + (1-t) x_0 = t\tilde{x}_1 + (1-t) \tilde{x}_0 ] \\ &\qquad\qquad\qquad\qquad\qquad\qquad - \mathbb{E}_{\tilde{x}_0',\tilde{x}_1' \sim q^k} [\tilde{x}_1' - \tilde{x}_0' | tx_1' + (1-t) x_0' = t\tilde{x}_1' + (1-t) \tilde{x}_0' ]\|^2 \big] \\ &\leq \mathbb{E}_{(x_0, x_1, x_0', x_1') \sim \tilde{q}^k} \big[ \big\| x_1 -x_0 - (x_1' -x_0') \big\|^2 \big] \leq 2\mathbb{E}_{(x_0, x_1, x_0', x_1') \sim \tilde{q}^k} \big[ \big\| x_1 -x_1' \big\|^2 + \big\| x_0 -x_0' \big\|^2 \big] \\ &= 2\mathbb{E}_{(x_0, x_1, x_0', x_1') \sim \tilde{q}^k} \big[ \big\| (x_0,x_1) -(x_0',x_1') \big\|^2 \big] = 2W_2^2(q,q^k) 
\end{split}
\end{align}
Assumption \textbf{(A3)} implies that almost surely, $q^k$ converges to $q$ weakly. For distributions on a bounded domain, weak convergence is equivalent to convergence in the Wasserstein distance \citep[Thm.~6.8]{villani2008optimal}, and this means that $W_2^2(q,q^k) \xrightarrow[]{k \to \infty} 0$ almost surely. 
Almost sure convergence implies convergence in probability, which means that
\begin{align}
    \mathrm{Pr}(W_2^2(q,q^k) > \epsilon) \xrightarrow[]{k \to \infty} 0, \quad \forall \epsilon > 0.
\end{align}
Note that $W_2^2(q,q^k)$ is a bounded random variable because $q$ and $q^k$ have bounded support as $q_0, q_1, q_0^k$ and $q_1^k$ have bounded support. Suppose that $W_2^2(q,q^k)$ is bounded by the constant $C$. Hence, we can write
\begin{align}
    &\mathbb{E}_{\bm{X}_0^k, \bm{X}_1^k} \mathbb{E}_{t \sim \mathrm{U}(0,1), (x_0, x_1, x_0', x_1') \sim \tilde{q}^k} \big[ \|u_t^{*}( t x_1 + (1-t) x_0 ) - u_t^{k}( t x_1' + (1-t) x_0' )\|^2 \big] \\ &\leq 2 \mathbb{E}_{\bm{X}_0^k, \bm{X}_1^k} W_2^2(q,q^k) \leq 2 \big( \epsilon + C \mathrm{Pr}(W_2^2(q,q^k) > \epsilon) \big).
\end{align}
We can take $\epsilon$ arbitrarily small, and for a given $\epsilon$ we can make the second term in the right-hand side arbitrarily small by taking $k$ large enough. The final result follows.
\end{proof}

\subsubsection{Proof of Proposition \ref{prop:almost_sure_weak_convergence}} \label{subsec:batchOT_works}

We have that almost surely, the empirical distributions $q_0^k$, resp. $q_1^k$, converge weakly to $q_0$, resp. $q_1$ \citep{varadarajan1958ontheconvergence}. Hence, we can apply Theorem~\ref{thm:cuestas_conv}. Since convergence in distribution of random variables is equivalent to weak convergence of their laws, and the law of an optimal coupling is the optimal transport plan, we conclude that $(q_k)_{k \geq 0}$ converges weakly to $q^*$.

\begin{theorem}[\cite{cuestaalberto1997optimal}, Theorem 3.2] \label{thm:cuestas_conv}
    Let ${(P_n)}_n$, ${(Q_n)}_n$, $P$, $Q$ be probability measures in $\mathcal{P}_2$ (the space of Borel probability measures with bounded second order moment) such that $P \ll \lambda_p$ ($P$ is absolutely continuous with respect to the Lebesgue measure) and $P_n \xrightarrow[]{w} P$, $Q_n \xrightarrow[]{w} Q$, where $\xrightarrow[]{w}$ denotes weak convergence of probability measures. Let $(X_n,Y_n)$ be an optimal coupling between $P_n$ and $Q_n$, $n \in \mathbb{N}$, and $(X,Y)$ an optimal coupling between $P$ and $Q$. Then, $(X_n, Y_n) \xrightarrow[]{\mathcal{L}} (X, Y)$, where $\xrightarrow[]{\mathcal{L}}$ denotes convergence of random variables in distribution.
\end{theorem}

\subsection{Bounds on the transport cost and monotone convergence results} \label{subsec:monotone}
The following result shows that for an arbitrary joint distribution $q(x_0,x_1)$, we can upper-bound the transport cost associated to the marginal vector field $u_t$ to a quantity that depends only $q(x_0,x_1)$.
\begin{proposition}\label{thm:transport_cost}
    For an arbitrary joint distribution $q(x_0,x_1)$ with marginals $q_0(x_0)$ and $q_1(x_1)$, let $\phi_t$ be the flow corresponding to the marginal vector field $u_t$. We have that 
    \begin{align} \label{eq:transport_upper}
        & \! \! \E_{q_0(x_0)} \|\phi_1(x_0) - x_0\|^2 \leq \E_{q(x_0,x_1)} \|x_1 - x_0\|^2,
    \end{align}
\end{proposition}
\begin{proof}
    We make use of the notation introduced in App.~\ref{sec:proof_limiting}.
    We will rely on the fact that for all $t \in [0,1]$, the random variable $t x_1 + (1-t) x_0$, with $(x_0, x_1) \sim q$ has the same distribution as the random variable $\phi_t(x_0)$, with $x_0 \sim q_0$. This is a direct consequence of Lemma~\ref{lem:marginals}. Using that $\phi_t$ is the flow for $u_t$ and Jensen's inequality twice, we have that
    \begin{align}
    \begin{split}
        &\E_{x_0 \sim q_0} \|\phi_1(x_0) - x_0\|^2 \\ &= \E_{x_0 \sim q_0} \bigg\|\int_0^1 u_s( \phi_s(x_0)) \, ds \bigg\|^2 \leq \E_{t \sim \mathrm{U}(0,1), x_0 \sim q_0} \| u_t( \phi_t(x_0)) \|^2 \\ &= \E_{t \sim \mathrm{U}(0,1), (x_0, x_1) \sim q} \| u_t( tx_1 + (1-t)x_0) \|^2 
        \\ &= \E_{t \sim \mathrm{U}(0,1), (x_0, x_1) \sim q} \| \E_{(x_0', x_1') \sim q} \left[ u_t( tx_1 + (1-t)x_0 | x_0', x_1') | tx_1 + (1-t)x_0 = tx_1' + (1-t)x_0' \right] \|^2  
        \\ &\leq \E_{t \sim \mathrm{U}(0,1), (x_0, x_1) \sim q} \E_{(x_0', x_1') \sim q} \left[ \|  u_t( tx_1 + (1-t)x_0 | x_0', x_1')  \|^2 | tx_1 + (1-t)x_0 = tx_1' + (1-t)x_0' \right] \\ &= \E_{t \sim \mathrm{U}(0,1), (x_0, x_1) \sim q} \E_{(x_0', x_1') \sim q} \left[ \|  x_1' - x_0'  \|^2 | tx_1 + (1-t)x_0 = tx_1' + (1-t)x_0' \right] \\ &= \E_{t \sim \mathrm{U}(0,1), (x_0, x_1) \sim q} \|  x_1 - x_0 \|^2
    \end{split}
    \end{align}
    as needed.
\end{proof}
Note that that the statement and proof of this proposition is equivalent to Theorem~3.5 of \cite{liu2022flow}, although the language and notation that we use is different, which is why we though convenient to include it.

For the case of BatchOT, the following theorem shows that the quantity in the upper bound of \eqref{eq:transport_upper} is monotonically decreasing in $k$. The combination of Proposition \ref{thm:transport_cost} and Theorem \ref{thm:monotone} provides a weak guarantee that for BatchOT, the transport cost should not get much higher when $k$ increases.
\begin{theorem}\label{thm:monotone}
    Suppose that Multisample Flow Matching is run with BatchOT. For clarity, we make the dependency on the sample size $k$ explicit and let\footnote{Note that here $q^{(k)} := q$ is a marginalized distribution and is different from $q^k$ defined in Step \ref{item:3}.} $q^{(k)}(x_0,x_1) := q(x_0,x_1)$, and $\phi_t^k(x_0) := \phi_t(x_0)$. Then, for any $k \geq 1$, we have that 
    \begin{align}
    \begin{split}
        & \! \! \E_{q_0(x_0)} \|\phi_1^{k}(x_0) - x_0\|^2 \leq \E_{q^{(k)}(x_0,x_1)} \|x_1 - x_0\|^2, \\ 
        &\! \! \E_{q^{(k+1)}(x_0,x_1)} \|x_1 - x_0\|^2 \leq \E_{q^{(k)}(x_0,x_1)} \|x_1 - x_0\|^2.
    \end{split}
    \end{align}
\end{theorem}
\begin{proof}
We write
\begin{align}
\begin{split}
        &\E_{t \sim \mathrm{U}(0,1), \bm{X}_0^{k+1},\bm{X}_1^{k+1}} \E_{(x_0, x_1) \sim q^{k+1}} \| x_1 - x_0 \|^2 = \frac{1}{k} \E_{t \sim \mathrm{U}(0,1), \bm{X}_0^{k+1},\bm{X}_1^{k+1}} \bigg[ \sum_{i=1}^{k} \| x_1^{(i)} - x_0^{(\sigma_{k+1}(i))} \|^2 \bigg] \\ &= \frac{1}{k} \frac{1}{k+1} \E_{t \sim \mathrm{U}(0,1), \bm{X}_0^{k+1},\bm{X}_1^{k+1}} \bigg[ \sum_{j=1}^{k+1} \sum_{i \in [k+1]\setminus \{j\}} \| x_1^{(i)} - x_0^{(\sigma_{k+1}(i))} \|^2 \bigg] \\ &\leq \frac{1}{k} \frac{1}{k+1} \E_{t \sim \mathrm{U}(0,1), \bm{X}_0^{k+1},\bm{X}_1^{k+1}} \bigg[ \sum_{j=1}^{k+1} \sum_{i \in [k+1]\setminus \{j\}} \| x_1^{(i)} - x_0^{(\sigma_{k}^{-j}(i))} \|^2 \bigg] \\ &= \E_{t \sim \mathrm{U}(0,1), \bm{X}_0^{k},\bm{X}_1^{k}} \bigg[ \frac{1}{k}\sum_{j=1}^{k} \| x_1^{(i)} - x_0^{(\sigma_{k}(i))} \|^2 \bigg] = \E_{t \sim \mathrm{U}(0,1), \bm{X}_0^{k},\bm{X}_1^{k}} \E_{(x_0, x_1) \sim q^{k}} \| x_1 - x_0 \|^2.
\end{split}
\end{align}
In the first equality, we used that the optimal transport map between the empirical distributions $q_0^k$ and $q_1^k$ can be encoded as a permutation, which we denote by $\sigma_{k+1}$. In the inequality, we introduced the notation $\sigma_{k}^{-j}$ to denote the optimal permutation within $\{x^{(i)}_0\}_{i \in [k+1]\setminus\{j\}}$. The inequality holds because using the optimality of $\sigma_{k+1}$:
\begin{align}
\begin{split}
&\sum_{j=1}^{k+1} \sum_{i \in [k+1]\setminus \{j\}} \| x_1^{(i)} - x_0^{(\sigma_{k+1}(i))} \|^2 \leq \sum_{j=1}^{k+1} \sum_{i \in [k+1]} \| x_1^{(i)} - x_0^{(\sigma_{k+1}(i))} \|^2 \\ &\leq \sum_{j=1}^{k+1} \bigg(\sum_{i \in [k+1] \setminus \{j\}} \| x_1^{(i)} - x_0^{(\sigma_{k}^{-j}(i))} \|^2 + \| x_1^{(j)} - x_0^{(j)} \|^2 \bigg) \leq \sum_{j=1}^{k+1} \sum_{i \in [k+1] \setminus \{j\}} \| x_1^{(i)} - x_0^{(\sigma_{k}^{-j}(i))} \|^2.
\end{split}
\end{align}
\end{proof}

\pagebreak
\section{Experimental \& evaluation details} \label{app:exp}

\begin{table}[H]
\centering
\begin{tabular}{l c c c}
\toprule
  & ImageNet-32 & ImageNet-64 \\
\midrule
Channels & 256 & 192  \\
Depth & 3 & 3 \\
Channels multiple & 1,2,2,2 & 1,2,3,4 \\
Heads & 4 & 4 \\
Heads Channels & 64 & 64 \\
Attention resolution & 4 & 8 \\
Dropout & 0.0 & 0.1 \\
Batch size / GPU & 256 & 50  \\
GPUs & 4 & 16   \\
Effective Batch size & 1024 & 800  \\
Epochs & 350 & 575 \\
Effective Iterations & 438k & 957k   \\
Learning Rate & 1e-4 & 1e-4 \\
Learning Rate Scheduler & Polynomial Decay & Constant  \\
Warmup Steps & 20k & -  \\
\bottomrule
\end{tabular}
\caption{Hyper-parameters used for training each model.}
\label{tab:hyper-params}
\end{table}

\subsection{Image datasets}

We report the hyper-parameters used in Table \ref{tab:hyper-params}. 
We use the architecture from \citet{dhariwal2021diffusion} but with much lower attention resolution.
We use full 32 bit-precision for training ImageNet-32 and 16-bit mixed precision for training ImageNet-64.  
All models are trained using the Adam optimizer with the following parameters: $\beta_1 = 0.9$, $\beta_2=0.999$, weight decay = 0.0, and $\epsilon = 1e{-8}$. 
All methods we trained using identical architectures, with the same parameters for the the same number of epochs (see Table \ref{tab:hyper-params} for details), with the exception of Rectified Flow, which we trained for much longer starting from the fully trained CondOT model.
We use either a constant learning rate schedule or a polynomial decay schedule (see Table \ref{tab:hyper-params}).  
The polynomial decay learning rate schedule includes a warm-up phase for a specified number of training steps. 
In the warm-up phase, the learning rate is linearly increased from $1e{-8}$ to the peak learning rate (specified in Table \ref{tab:hyper-params}). 
Once the peak learning rate is achieved, it linearly decays the learning rate down to $1e{-8}$ until the final training step.

When reporting negative log-likelihood, we dequantize using the standard uniform dequantization \citep{dinh2016density}. We report an importance-weighted estimate using
\begin{equation}
    \text{BPD}(K) = -\frac{1}{D} \log_2 \frac{1}{K} \sum_{k=1}^K p_t (x + u_k), \text{ where } u_k \sim [U(0, 1)]^D,
\end{equation}
with $x$ is in $\{0, \dots, 255\}^D$. We solve for $p_t$ at exactly $t=1$ with an adaptive step size solver \texttt{dopri5} with \texttt{atol=rtol=1e-5} using the \texttt{torchdiffeq}~\citep{torchdiffeq} library. We used $K$=15 for ImageNet32 and $K$=10 for ImageNet64.

When computing FID, we use the TensorFlow-GAN library \url{https://github.com/tensorflow/gan}.

We run coupling algorithms only within each GPU. We also ran coupling algorithms across all GPUs (using the ``Effective Batch Size'') in preliminary experiments, but did not see noticeable gains in sample efficiency while obtaining slightly worse performance and sample quality, so we stuck to the smaller batch sizes for running our coupling algorithms.

For Rectified Flow, we use the finalized FM-CondOT model, generate 50000 noise and sample pairs, then train using the same FM-CondOT algorithm and hyperparameters on these sampled pairs. This is equivalent to their 2-Rectified Flow approach \citep{liu2022flow}. For the rectification process, we train for 300 epochs. 

\subsection{Improved batch optimal couplings} \label{app:exp_stat_vs_dyn}
\textbf{Datasets.} \; We experimented with 3 datasets in dimensions $\{2,32,64\}$ consisting of $50K$ samples. Both $q_0$ and $q_1$ were Gaussian mixtures with number of centers described in Table \ref{tab:hyp_stat}. 

\textbf{Neural Networks Architectures.} \; For B-ST we used stacked blocks of Convex Potential Flows \cite{huang2020convex} as an invertible neural network parametrizing the map, which also allowed us to estimate KL divergence:
\begin{equation}
    \text{KL}(q_1 || \left(\psi_1\right)_{\sharp}q_0) = \E_{x \sim q_1} \left[ \log q_1(x) - \log \left((\psi_1)_{\sharp}q_0\right)(x) \right].
\end{equation}
For B-FM we used a simple MLP with Swish activation. For each dataset we built architectures with roughly the same number of parameters. 

\textbf{Hyperparameter Search.} \; For each dataset and each cost we swept over learning rates $\{0.005, 0.001, 0.0005\}$ and chose the best setting.

\begin{table}[H] 
\centering
\begin{tabular}{c|ccc}
                                     & 2-D & 32-D & 64-D \\ \hline
\multicolumn{1}{l|}{$q_0$ \#centers} & 1   & 50   & 100  \\
\multicolumn{1}{l|}{$q_1$ \#centers} & 8   & 50   & 100  \\
\multicolumn{1}{l|}{\#params}        & 50K & 800K & 800K \\
\multicolumn{1}{l|}{batch size}      & 128 & 1024  & 1024 \\
\multicolumn{1}{l|}{epochs}      & 100 & 1000  & 1000 
\end{tabular}
\caption{Hyperparameters for experiments on synthetic datasets.}
\label{tab:hyp_stat}
\end{table}

\end{document}